\documentclass[onecolumn,draftclsnofoot]{IEEEtran}
\usepackage{subfigure}
\usepackage{epsfig,graphicx,psfrag,dsfont}
\usepackage{amsfonts,amsmath,amssymb,color}
\usepackage{bbm,setspace,amsthm,cite}
\usepackage[linesnumbered,ruled]{algorithm2e}
\usepackage{authblk}
\usepackage{xspace}
\usepackage{bbm}
\def\var{\mathop{\rm var}\nolimits}%
\def\diag{\mathop{\rm diag}\nolimits}%
\def\rank{\mathop{\rm rank}\nolimits}%
\newcommand{\Bc}{\mathcal{B}}
\newcommand{\Cc}{\mathcal{C}}

\newcommand{\Ec}{\mathcal{E}}

\newcommand{\Nc}{\mathcal{N}}

\newcommand{\Xv}{{\bf X}}

\newcommand{\xv}{{\bf x}}
\newcommand{\av}{{\bf a}}

\newcommand{\Av}{{\bf A}}
\newcommand{\Bv}{{\bf B}}
\newcommand{\Cv}{{\bf C}}
\newcommand{\mv}{{\bf m}}

\newcommand{\sh}{{\hat{s}}}

\def\g{\gamma}

\def\e{\epsilon}

\def\l{\lambda}

\DeclareMathOperator\E{E}
\let\P\relax
\DeclareMathOperator\P{P}
\newcommand\ie{i.e.,\xspace}
\def\textiid{i.i.d.\@\xspace}
\newcommand\iid{\ifmmode\text{ i.i.d. } \else \textiid \fi}

\newcommand{\ind}{\mathbbmss{1}}

\newcommand{\Dhmu}{\hat{\delta}_{\mu}}
\newcommand{\dmu}{\delta_{\mu}}
\newcommand{\muh}{\hat{\mu}}

\title{Linear Time Clustering for High Dimensional Mixtures of Gaussian Clouds}
\author{Dan Kushnir\thanks{Corresponding Author: dan.kushnir@nokia-bell-labs.com}}
\author{Shirin Jalali\thanks{shirin.jalali@nokia-bell-labs.com}}
\author{Iraj Saniee\thanks{iraj.saniee@nokia-bell-labs.com}}
\affil{Bell Laboratories, Nokia, Murray Hill, NJ 07974}
\date{\today}

\begin{document}

\maketitle
\newtheorem{claim}{Claim}
\newtheorem{theorem}{Theorem}
\newtheorem{lemma}{Lemma}
\newtheorem{remark}{Remark}
\newtheorem{corollary}{Corollary}

\begin{abstract}
Clustering mixtures of Gaussian distributions is a fundamental and challenging problem that is ubiquitous in various high-dimensional data processing tasks. While state-of-the-art work on learning Gaussian mixture models has focused primarily on improving separation bounds and their generalization to arbitrary classes of mixture models, less emphasis has been paid to practical computational efficiency of the proposed solutions. In this paper, we propose a novel and highly efficient clustering algorithm for  $n$ points drawn from a mixture of two arbitrary Gaussian distributions in $\mathbb{R}^p$. The algorithm involves performing random 1-dimensional projections until a direction is found that yields a user-specified clustering error $e$. For a 1-dimensional separation parameter $\gamma$ satisfying $\gamma=Q^{-1}(e)$, the expected number of such projections is shown to be bounded by $o(\ln p)$, when $\gamma$ satisfies $\gamma\leq c\sqrt{\ln{\ln{p}}}$, with $c$ as the separability parameter of the two Gaussians in $\mathbb{R}^p$. Consequently, the  expected overall running time of the algorithm is linear in $n$ and quasi-linear in $p$ at $o(\ln{p})O(np)$, and the sample complexity is independent of $p$. This result stands in contrast to prior works which provide polynomial, with at-best quadratic, running time in $p$ and $n$.  We show that our bound on the expected number of 1-dimensional projections extends to the case of three or more Gaussian components, and we present a generalization of our results to mixture distributions beyond the Gaussian model.
\end{abstract}

\section{Introduction}

Consider the problem of clustering a mixture of two Gaussian distributions in $\mathds{R}^p$, based on $n$ random drawings   $\Xv_1,\ldots,\Xv_n$. Each point $\Xv_j$, with probability $w_1$,  is drawn from $\Nc(\mv_1,\Sigma_1)$,  and with probability $w_2=1-w_1$ from $\Nc(\mv_2,\Sigma_2)$, where $\mv_i\in\mathds{R}^p$ and $\Sigma_i\in\mathds{R}^{p\times p}$, for $i=1,2$.  Given unlabeled observation points  $\Xv_1,\ldots,\Xv_n$ in $\mathds{R}^p$, the clustering task aims at labeling each point in $\mathds{R}^p$ as either 1 or 2. The goal is to minimize the clustering error probability, which is defined as the probability that  the label of the  Gaussian that has generated a point disagrees with its assigned label, up to a fixed permutation of all labels.

%The main focus of the research on learning parameters of Gaussian Mixture Model (GMM) has been on using clustering as means for estimating the parameters, where only in few cases the estimated parameters are used as a step towards assigning data points to clusters. In such cases, once the parameters are learned, then as a byproduct,
The main focus of early research on learning Gaussian Mixture Models (GMMs) has been on estimating the parameters of the GMM first and then using the estimated parameters as a step towards assigning data points to clusters. In such cases, once the parameters are learned with sufficient accuracy, then as a byproduct, one can cluster points by assigning to each point the Gaussian cloud with highest posterior probability. Of high popularity in this context are the EM and K-means algorithms \cite{Dempster,LloydKmeans}, whose convergence has weak guarantees. These shortcomings are especially critical at high dimensions, where the number of parameters to be estimated is quadratic in the dimension $p$.
%In particular, the estimation of the $p \times p$ component covariance matrices (especially in high dimensional spaces) is very challenging, both theoretically and algorithmically.
In particular, the estimation of the $p \times p$ component covariance matrices is very challenging, especially in high dimensions both theoretically and algorithmically.
%When the sample size $n$ is comparable with the ambient dimension $p$, and not with $p^2$, while clustering of the points might be feasible, without any further knowledge about the structure of the Gaussians, estimating the covariance matrices becomes almost impossible. To this end, not much research has been dedicated to the effect of parameters estimation error on the actual clustering error.
When the sample size $n$ is comparable with the ambient dimension $p$, and not with $p^2$, while clustering of the points might be feasible, estimating the covariance matrices becomes almost impossible without any further knowledge about the structure of the Gaussians. To alleviate these difficulties many approaches for learning GMMs employ random projections for embedding the data in a lower dimensional subspace where parameter learning can be performed more efficiently. %\cite{Dasgupta:99,VempalaWang,Kenan:05,Sinha1,kalai2010efficiently,Moitra,Sinha1}.}

%We provide a comprehensive table on related work on GMM parameter learning and their complexity in table \ref{tab:relwrk}.
In Table \ref{tab:relwrk} we provide a summary of prior work on GMM parameter learning, clustering and their computational complexity.
A clear trend is visible from the table in which most of the research has concentrated on improving the performance with respect to parameter estimation for decreasing component separation and on generalizing the achieved bounds to arbitrary GMMs. On the other hand, the running-time complexity of these methods  was addressed only as polynomially bounded for most methods, and further algorithmic analysis shows at least quadratic running time in the dimension or sample size. Thus, the majority of these methods are inapplicable for very high dimensional and big data. An exception of this is the algorithm suggested in \cite{schulman} for learning a spherical GMM. Most notably, it works in the ambient dimension with a modified 2-round EM algorithm using auxiliary procedures for pruning overlapping and small clusters. However, this method is restricted to well-separated spherical GMMs only, while for a non-spherical GMM EM requires at least a $p^2$ computation of the covariance matrix.
\begin{table}[t]\scriptsize
\begin{tabular}{ |p{12mm}|p{13mm}|p{13mm}|p{32mm}|p{21mm}|p{6mm}|p{26mm}| }
 \hline
 Author & Method & GMM Class & Complexity  & Sample Complexity & Min. means Sep. & Parameters\textbackslash comments  \\\hline
  \cite{Dasgupta:99} & Random projection &Shared Spherical Covariance & $O(dn^2+ndp)$ & $k^{O(\log^2(1/{\varepsilon \delta}))}$ & $\sqrt{p}$& $d$- num. projections, $\varepsilon:  \{\|\hat{\mu}_i-\mu_i \|_{L_2} \}\leq \varepsilon \sigma_{max}\sqrt{p}$  \\\hline
    \cite{schulman} & EM & Spherical GMM& $npk$ & $\Omega(l max(1; c^{-2}))$ & $\Omega (p^{\frac{1}{4}})$ & $l=\Omega(\frac{1}{w_{min}}\ln{\frac{1}{w_{min}}})$ \\\hline
  \cite{Arora} & Distance based & Arbitrary GMM & $O(p^2poly(k) \log^2 {\frac{p}{\delta}})$ and $O(pn^2)$ for distance computation & $O(p \cdot poly(k) \log{\frac{p}{\delta}})$& $\Omega(p^{\frac{1}{4}})$ & $k$ - num. Gaussians\\\hline
 \cite{VempalaWang} & Spectral, Distance & Spherical  GMM& $poly(p,k)$, $p^3$ for SVD and $O(pn^2)$ for distance computation& $poly(p,k)$ & $\Omega(k^{\frac{1}{4}})$ & \\\hline
 \cite{Kenan:05} & Spectral & log-concave & $poly(p,\frac{1}{\epsilon},\log{\frac{1}{\delta}})$, $p^3$ for SVD& $O(\frac{p}{\epsilon}\log^3{\frac{p}{\delta}})$ & $ \frac{k^{\frac{3}{2}}}{\epsilon^2} $ & need to know $w_i$, $\epsilon=w_{min}$ \\\hline
%Achlioptas and McSherry \cite{Achiloptas} & Spectral & Arbitrary & $p^3$ (spectral decomposition) &  & & & \\\hline
\cite{Feldman}& Method of moments (MoM) & Axis Aligned GMM & $poly(\frac{p}{\epsilon})$&  & $>0$& Learns density (not params.), $\epsilon$: $L_1$ distance of means and $w_i$'s \\\hline
 \cite{Sinha1} & Spectral & Identical spherical GMM & $\frac{k^{\frac{3}{2}}}{(\sigma_{min}^4\sigma)^4}\left(\frac{p^{3/ 2}k^{1/ 2}}{\epsilon}\right)^{C_1 k^2} $& $poly((\frac{pk^2}{\epsilon})^{k^3}\cdot \log^k{\frac{2}{\delta}})$ & $>0$ & \\\hline
\cite{kalai2010efficiently} & Random projection and MoM & Arbitrary 2-GMM &  $poly(p,\frac{1}{\epsilon},\frac{1}{\delta},\frac{1}{w},\frac{1}{D_{1,2}})$, $p^2$ projections & same as running time& $\geq 0$ & $D_{1,2}=D(F_1,F_2)$ - statistical distance with $\epsilon$ as its accuracy param. \\\hline
 \cite{Moitra} & Random projection and MoM & Arbitrary GMM & same as \cite{kalai2010efficiently} & same as \cite{kalai2010efficiently} &$\geq 0$  &   \\\hline
 \cite{Sinha2} & Deterministic projection MoM & Arbitrary GMM & $poly(p,\frac{1}{\epsilon},\frac{1}{\delta},B)$, algorithm uses $\binom{p}{2k^2}$ projections & $poly(p,\frac{1}{\epsilon},\frac{1}{\delta},B)$  & $\geq 0$ & $\epsilon$ - $L_2$ error in params. $B$ - radios of params. identifiability ball\\\hline
Our Algorithm & Random Projection and MoM & Arbitrary 2-mixture of distributions & expected $o(\ln{p})O(np)$ when $\frac{\gamma}{c}\leq \sqrt{\ln{\ln{p}}}$ & $O(\frac{1}{\epsilon^2}\log{\frac{1}{\delta}}) $& $\sqrt{p}$ & $\epsilon$ - error in params., $\gamma$ - separation in 1-dimension s.t. $Q^{-1}(e)=\gamma$ for clustering error $e$, $Q$ - Q-function\\\hline
\end{tabular}
\label{tab:relwrk}
\caption{Related GMM learning methods.}
\end{table}
%\end{center}
\normalsize

More recent approaches learn the parameters of a GMM by projecting the data points to a lower dimensional space to estimate
the labels of the points, and then learn the parameters of the Gaussians in
the high dimensional space using the estimated labels \cite{Dasgupta:99,VempalaWang,Kenan:05,Sinha1,kalai2010efficiently,Moitra,Sinha2}. The number of projections as well as the
learning of the parameters in the projected space have been shown to be possible
in polynomial time in the GMM parameters with algorithms at least quadratic in the dimension $p$ and$/$or number $n$
of samples. However, it is an open question if one can learn clusters
accurately enough in linear time in the ambient dimension $p$ or
the sample size $n$, before or after estimating the Gaussian parameters
in the ambient or lower dimensional sub-spaces.
This question has motivated our work to explore the tradeoff between the accuracy
of clustering in collections of 1-dimensional random projections and the
running time. Surprisingly enough, we show that when a user-specified
clustering error reflects a 1-dimensional separability close to the separation of the Gaussians in $\mathds{R}^p$,
our proposed algorithm achieves the desired accuracy by using a number of random  projections that is sub-logarithmic in $p$.
With a 1-dimensional learner (such as one based on the method of moments (MoM) \cite{pearson1894contributions} or expectation maximization (EM) \cite{Dempster}) that runs in $O(n)$ time, we achieve an expected running time that is quasi-linear in $p$ and linear in $n$, at $o(\ln{p})O(np)$. This running time is achieved when the prescribed clustering error $e$ corresponds to a 1-dimensional separability $\gamma$ such that $\gamma\leq c\sqrt{\ln{\ln{p}}}$, where $c$ is defined as
\begin{align}
\label{eq:sep_intro}
c={\|\mv_1-\mv_2\|\over  \sqrt{p} \left(\sqrt{\lambda_{\max}(\Sigma_1)}+\sqrt{\lambda_{\max}(\Sigma_2)}\right)},
\end{align}
and $\lambda_{\max}(\Sigma_i)$ denotes the the maximum eigenvalue of $\Sigma_i$, $i=1,2$.

%Our results are motivated by an initial observation that a random projection into 1-dimension not only  preserves separability $c$ in expectation (as one may hope for), but also, with a probability that depends on $c$, and the ambient dimension  $p$, can be larger than $c$.
Our results are motivated by the observation that a random projection into 1-dimension not only preserves the separability $c$ in expectation (as one may hope for), but with a probability that depends on $c$ and the ambient dimension $p$, may be even larger than $c$. Thus, learning parameters in 1-dimension is not only computationally (much) faster but is also easier as the sample concentration grows in 1-dimension.
The above can be seen clearly in the following empirical distribution of the 1-dimensional separability values $\gamma$ obtained by projecting a mixture of two Gaussians with separability $c$ into multiple 1-dimensional random directions and registering their separability values. As seen in Fig.~\ref{fig:prob_sep}, a significant fraction of the probability mass lies in values higher than $c$. Recovering the random directions in which the separability is higher than $c$ can be done remarkably fast by sequentially scanning each of them with a learning algorithm that runs in 1-dimension in $O(n)$ time.

\begin{figure}[!htpb]
\centering
\includegraphics[width=3in]{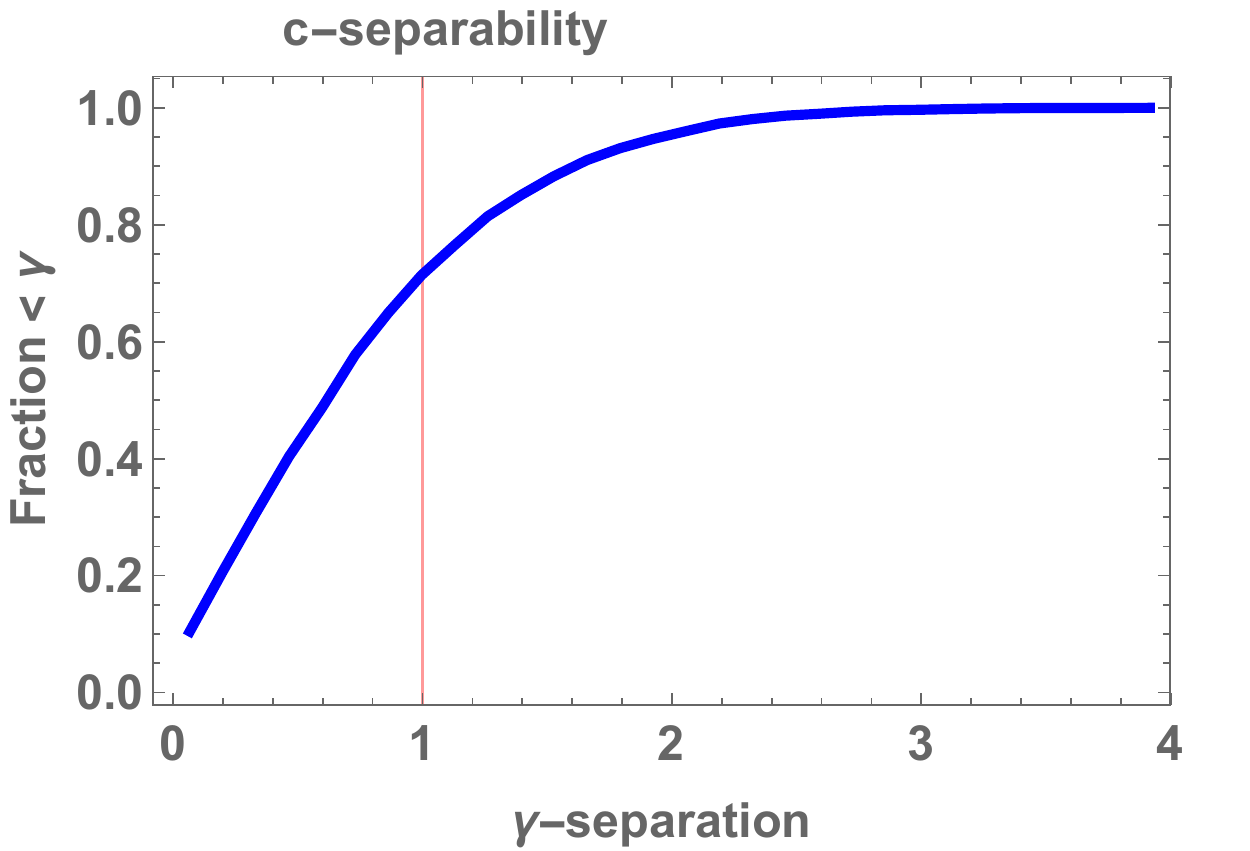}
\caption{\scriptsize{Cumulative distribution of $\gamma$ generated from random projections of spherical Gaussians mixtures in $\mathds{R}^{1000}$ with $c=1$.}}
 \label{fig:prob_sep}
\end{figure}

%
%\begin{figure}[!htpb]
%\begin{center}
%    \psfig{figure=./sepH2Lo_30bins_sep1_1000D_norm_hist.eps,height=50mm,width=55mm}
%%        \psfig{figure=./cdf_gamma.eps,height=50mm,width=55mm}
%\caption{\scriptsize{Empirical probability of $\gamma$ generated from random projections of spherical Gaussians mixtures in $\mathds{R}^{1000}$ with $c=1$.}}
% \label{fig:prob_sep}
% \end{center}
% \end{figure}
The two observations above, namely, non-negligible probability of random directions with $\gamma>c$ and $O(n)$ efficiency in 1-dimensional parametric clustering give rise to running time that is linear in $n$ and quasi-linear in $p$ for clustering GMMs via 1-dimensional projections. This complexity is a significant running time improvement over state-of-the-art algorithms that cluster GMMs by learning their parameters. For instance, the number of 1-dimensional projections is at best quadratic in $p$ in \cite{Moitra,kalai2010efficiently}, or cubic in $p$ for spectral-projections-based methods \cite{VempalaWang,Kenan:05,Achiloptas,Sinha1}.  In other cases, where higher order random projection subspaces are used, the running time is quadratic in $n$ \cite{Dasgupta:99}, similarly to the quadratic complexity achieved in distance-based methods \cite{Arora}. Clearly, the purpose of these methods is to learn the parameters to a high accuracy, even for arbitrary separation. On the other hand, for clustering, we show that one can generate accurate clustering in time linear in $n$ and quasi-linear $p$, with separation that scales as $\sqrt{p}$, similar to the minimal separation required in \cite{Dasgupta:99}. For example, data sets such as the handwritten digits of USPS \cite{uci} in $[-1,1]^{256}$ are 0.63-separated.

We generalize our results to the case of $k>2$, providing the dependance of the number of projections to achieve separability $\gamma$ on the number of components. Moreover, since our result on the number of projections does not rely on the Gaussianity of the data, we use the central limit theorem (with respect to $p$) to show that our analysis readily extends to mixtures of two arbitrary distributions of finite second moment. %In particular, by using the randomness of the projection vector in the central limit theorem, the projected data converges to the Gaussian distribution. This convergence allows similar mapping from separability in 1-dimension to the clustering error as in the Gaussian case.

 We analyze the sample complexity of our proposed algorithm and its impact on the clustering accuracy (see Appendix \ref{sec:smpl_complex}). Our sample complexity is $O(\frac{1}{\epsilon^2}\log{\frac{1}{\delta}})$, where $\epsilon$ is the error in the GMM's parameter estimation, and $\delta$ is the confidence. This result relies on solving the 1-dimensional parameter learning problem using MoM following on \cite{price}. Thus, we alleviate the sample complexity dependency on the dimension, where in other GMM learning methods (\cite{Arora,VempalaWang,Kenan:05,Sinha1,kalai2010efficiently,Sinha2,Moitra}) the sample complexity depends on $p$.

We provide our main results in Section \ref{sec:sepAfterProj} whereas for space-limit reasons most of the main proofs are in the relevant Appendix \ref{sec:proofs}. In Section \ref{sec:arbitrary_dist} we discuss the extension of our results to mixtures of arbitrary distributions. Our algorithm is presented in Section \ref{sec:alg}, and its sample complexity is analyzed in Appendix \ref{sec:smpl_complex}. Finally, we provide empirical validation to our algorithm's complexity and accuracy bounds in Appendix \ref{sec:experiments}.

%------------%------------%------------%-------------%------------%------------%-------------%
\section{Background}
Consider the problem of clustering  $n$ points $\Xv_1,\ldots,\Xv_n$  in $\mathds{R}^p$ that are  generated according to a mixture of two Gaussian  distributions as $w_1\Nc(\mv_1,\Sigma_1)+w_2\Nc(\mv_2,\Sigma_2)$. A  clustering algorithm $\Cc: \mathds{R}^p\to\{1,2\}$, without having access to the  parameters $(\mv_1,\mv_2,\Sigma_1,\Sigma_2,w_1)$, maps each point in $\mathds{R}^p$ into label $1$ or  label $2$. For $\Xv\in\mathds{R}^{p}$ generated according to the described distribution, let $T(\Xv)$ denote the random variable that represents the underlying correct label.
Let $\e(\alpha,\mv_1,\mv_2,\Sigma_1,\Sigma_2)$ denote the  error of an optimal {\em Bayesian} classifier of the mixture of two Gaussians, which has access to the parameters of the Gaussians and their weights. Then $\e(\alpha,\mv_1,\mv_2,\Sigma_1,\Sigma_2)$  can be bounded via the $Q$-function as follows:
\begin{theorem}\label{thm:c-sep-vs-error}
Consider points  in $\mathds{R}^p$ drawn from a mixture of two Gaussian distributions   $\alpha\Nc(\mv_1,\Sigma_1)+(1-\alpha)\Nc(\mv_2,\Sigma_2)$ . Assume that the two components of the mixture are $c$-separable. Then,
\begin{align}
\label{eq:opt_clust_err}
\e(\alpha,\mv_1,\mv_2,\Sigma_1,\Sigma_2) &\leq Q\left({c\over 2}\sqrt{p}\right),
\end{align}
where
\begin{align}
Q(x)={1\over \sqrt{2\pi}}\int^{\infty}_x {\rm e}^{-{u^2\over 2}}du.\label{eq:def-Q}
\end{align}
\end{theorem}

%------------%------------%------------%-------------%------------%------------%-------------%
\section{Main Results}
\label{sec:sepAfterProj}
In this section we consider data that is generated according to a mixture of two  Gaussian distributions  $\Nc(\mv_1,\Sigma_1)$ and $\Nc(\mv_2,\Sigma_2)$, which are $c$-separable, i.e.
\begin{align}
\label{eq:sep_sec}
{\|\mv_1-\mv_2\|\over  \sqrt{p} \left(\sqrt{\lambda_{\max}(\Sigma_1)}+\sqrt{\lambda_{\max}(\Sigma_2)}\right)}\geq c.
\end{align}
  We study the probability that a random projection achieves a 1-dimensional separability $\gamma$ or higher, which can be directly related to a prescribed clustering error $e$ as $e\leq Q(\gamma)$ for two Gaussian distributions. Moreover, we provide conditions for the number of 1-dimensional projections required to achieve separability $\gamma$ to be sub-logarithmic in $p$ when $\gamma$ (corresponding to a clustering error in 1-dimension) is similar to $c$. These results allow the construction of very efficient clustering algorithms that run in $o(\ln{p})O(np)$.

  We divide the discussion into two cases. The first case is when the two Gaussians are spherical balls. The second case is when $\Sigma_1$ and $\Sigma_2$ are arbitrary positive semi-definite matrices. We also demonstrate the extension of our theoretical analysis for a mixture of $k$ Gaussians.

\subsection{Mixture of spherical Gaussians}
\label{subsec:spherical_GMM}
Consider the special case where $\Sigma_i=\sigma_i^2I_p$, for $i=1,2$.  We examine the projections of points generated as $w_1\Nc(\mv_1,\Sigma_1)+w_2\Nc(\mv_2,\Sigma_2)$ using a random vector $\Av=(A_1,\ldots,A_p)$, where $A_1,\ldots,A_p$ are i.i.d. as $\Nc(0,1)$. Using this projection, we derive a mixture of two Gaussians in $\mathds{R}$.
%\[
%\Xc_p=\{\langle \Xv_i,\Av\rangle:  i=1,\ldots,|\Xc|\}.
%\]
%Similarly, each point in $\Yc_p=\{\langle \Yv_i,\Av\rangle:  i=1,\ldots,|\Yc|\}$.
Conditioned on $\Av=\av$,  the two Gaussians $\Nc(\mv_1,\Sigma_1)$ and $\Nc(\mv_1,\Sigma_1)$ in $\mathds{R}^p$ are mapped to $\Nc(\langle \mv_1,\av\rangle,\sigma_1^2\|\av\|^2)$, and $\Nc(\langle \mv_2,\av\rangle,\sigma_2^2\|\av\|^2)$, respectively.
Therefore,   the two clusters are  $\gamma$-separable after projection, if $|\langle\mv_1,\av\rangle-\langle\mv_2,\av\rangle|> \gamma (\sigma_1+\sigma_2)\|\av\|,$
or
\begin{align}
|\langle\mv_1-\mv_2,\av\rangle|> \gamma (\sigma_1+\sigma_2)\|\av\|.\label{eq:sep-condition-iid}
\end{align}
%where $\gamma>0$ is a parameter that depends on the desired level of separability.

Since $\Av$ is not a fixed vector, the question is that given the randomness in the generation of the projection vector $\Av$, what is the probability that condition \eqref{eq:sep-condition-iid} holds. In other words, given $\mv_1$, $\mv_2$, $\sigma_1$ and $\sigma_2$, we are interested in
 \[
 \P(|\langle\mv_1-\mv_2,\Av\rangle|> \gamma(\sigma_1+\sigma_2)\|\Av\|),
 \]
 or
  \[
 \P\Big(\Big|\langle{\mv_1-\mv_2\over \|\mv_1-\mv_2\|},{\Av\over \|\Av\|} \rangle\Big|> {\gamma(\sigma_1+\sigma_2)\over \|\mv_1-\mv_2\|}\Big),
 \]
 where $A_1,\ldots,A_p\stackrel{\rm i.i.d.}{\sim}\Nc(0,1)$. The following theorem derives a lower bound on this probability.

\begin{theorem}\label{thm:1}
Consider $\mv_1,\mv_2\in\mathds{R}^p$ and $\sigma_1,\sigma_2\in\mathds{R}^+$. Assume that $\Av=(A_1,\ldots,A_p)$ are generated $i.i.d.$~according to $\Nc(0,1)$. Given $\gamma>0$, let
\[
\alpha \triangleq {\gamma^2(\sigma_1+\sigma_2)^2 p\over \|\mv_1-\mv_2\|^2}.
\]
Then, for any $\tau>0$
\begin{equation}
\label{eq:thm1}
\P\Big(|\langle\mv_1-\mv_2,{\Av\over \|\Av\|}\rangle|\geq  \gamma (\sigma_1+\sigma_2)\Big)\geq  \P\Big({A_1^2} >\alpha {(1-{1 \over p})\over (1-{\alpha\over p})}(1+\tau)\Big)(1- {\rm e}^{-{p-1\over 2}(\tau-\log(1+\tau))}).
\end{equation}
\end{theorem}

As shown by the following lemma, for two spherical $c$-separable Gaussian distributions in $\mathds{R}^p$, the expected value of the squared separability of the randomly  projected Gaussian distributions  in 1-dimension  is equal to $c^2$.
  \begin{lemma}\label{lemma:5}
Consider $\mv_1,\mv_2\in\mathds{R}^p$ and $\sigma_1,\sigma_2\in\mathds{R}^+$ and let
\[
c\triangleq {\|\mv_1-\mv_2\| \over (\sigma_1+\sigma_2)\sqrt{p}}.
\]
 Then, under a random 1-dimensional projection with  $\Av=(A_1,\ldots,A_p)$, where $A_1,\ldots,A_p\stackrel{\rm i.i.d.}{\sim}\Nc(0,1)$,
 \[
\E\left[ {|\langle \Av,\mv_1-\mv_2\rangle|^2\over (\sigma_1+\sigma_2)^2\|\Av\|^2 }\right]=c^2.
 \]
\end{lemma}
%-------------------%-------------------%-------------------%-------------------%-------------------%-------------------%-------------------%-------------------
Lemma \ref{lemma:5} is later used to approximate the unknown separation $c$ from the empirical expectation $E[\hat{\gamma}]$ in 1-dimensional projections. Having both $\gamma=Q^{-1}(e)$ and an estimate $\hat{c}$ for $c$, one can compute the probability of attaining $\gamma$ via Theorem \ref{thm:1}.
Finally, the following corollary, which is a direct consequence of Theorem \ref{thm:1},  characterizes the probability that two Gaussians with $\Sigma_1=\Sigma_2=\sigma^2I_p$ that are $c$-separable in $\mathds{R}^p$ remain $\gamma$-separable  in 1-dimension  under a random projection.

\begin{corollary}\label{cor:prob_csep}
Consider two spherical Gaussian distributions in $\mathds{R}^p$, with means $\mv_1,\mv_2\in\mathds{R}^p$ and covariance matrices  $\Sigma_1=\Sigma_2=\sigma^2 I_p$. Let $c={\|\mv_1-\mv_2\| / 2\sqrt{p\sigma^2}}.$ Assume that $\Av=(A_1,\ldots,A_p)$ are generated $i.i.d.$~according to $\Nc(0,1)$. Then, the probability that the 1-dimensional projected Gaussian distributions are $\gamma$-separable can be lower bounded as
\begin{equation}
\label{eq:limit}
\P\left(|\langle\mv_1-\mv_2,{\Av\over \|\Av\|}\rangle|\geq   2 \gamma \sigma\right)\geq  \P\left({A_1^2} > {(1-{1 \over p})\gamma^2 \over (c^2-{\gamma^2\over p})}(1+\tau)\right)\left(1- {\rm e}^{-{p-1\over 2}(\tau-\log(1+\tau))}\right),
\end{equation}
where  $\tau>0$ can be selected freely.
\end{corollary}
\begin{proof}
  Note that in Theorem \ref{thm:1}, $\alpha = {{\gamma }^2(2\sigma)^2 p\over \|\mv_1-\mv_2\|^2}={\gamma^2\over c^2}.$
\end{proof}
 The probability derived in (\ref{eq:thm1}) or (\ref{eq:limit}) can be used to calculate the expected number of projections to be examined until $\gamma$ is attained.

\begin{corollary}\label{cor:num_proj4gamma}
Consider two spherical Gaussian distributions in $\mathds{R}^p$, that are $c$-separable.
Denote $d(\gamma)$ as the average number of 1-dimensional random projections required to attain $\gamma$-separation in 1-dimension. Then as $p\rightarrow \infty$ we obtain $d(\gamma)\leq \frac{1}{2Q(\frac{\gamma}{c})}$.
\end{corollary}
\begin{proof}
We first note that in (\ref{eq:limit}) $\tau>0$, the term $1- {\rm e}^{-{p-1\over 2}(\tau-\log(1+\tau))}{\rightarrow} 1$, as ${p \rightarrow \infty}$. Next, in the limit we observe
 \begin{equation}
 \label{eq:prob_limit}
\lim_{p \rightarrow \infty} \P\left({A_1^2} > {(1-{1 \over p})\gamma^2 \over (c^2-{\gamma^2\over p})}(1+\tau)\right){\rightarrow} \P\left({A_1^2} > {\big(\frac{\gamma^2}{c^2}\big)}(1+\tau)\right).
\end{equation}
Therefore, for every $\tau>0$, from (11),
\[
\lim_{p\to\infty} d(\gamma)\leq {1\over \P\left({A_1^2} > {\big(\frac{\gamma^2}{c^2}\big)}(1+\tau)\right)}.
\]
Since $\tau$ is a free parameter, choosing it arbitrary close to zero yields the desired result. That is,
\[
\lim_{p\to \infty}d(\gamma)\leq {1\over 2Q({\gamma\over c})}.
\]
\end{proof}
In the following corollary we establish the conditions on $\gamma$ and $c$ so that with a number of projections that is sub-logarithmic in $p$ $\gamma$ can be achieved.
\begin{corollary}
\label{lema:runtime_sphr_olog}
 Consider two spherical Gaussian distributions in $\mathds{R}^p$, with means $\mv_1,\mv_2\in\mathds{R}^p$ and covariance matrices  $\Sigma_1=\Sigma_2=\sigma^2 I_p$.  Let  $d(\gamma)$ denote  the expected number of projections required  to achieve $\gamma$-separability. If $\gamma$ is such that  $\frac{\gamma}{c}=(\ln{\ln{p}})^{1-\eta\over 2}$,where $\eta>0$ is a free parameter,  then $d(\gamma)= o(\ln{p})$.
\end{corollary}
\begin{proof}
By Corollary \ref{cor:prob_csep}, choosing $p$ large enough so that ${\rm e}^{-{p-1\over 2}(\tau-\log(1+\tau))}\leq {1\over 2}$, it follows that
\[
d(\gamma)\leq {2\over \P\left({A_1^2} > {(1-{1 \over p})\gamma^2 \over (c^2-{\gamma^2\over p})}(1+\tau)\right)}={1\over Q\left(\sqrt{{(1-{1 \over p})\gamma^2 \over (c^2-{\gamma^2\over p})}(1+\tau)}\right)}.
\]
On the other hand, for all $x>0$, we have
\begin{equation}
\label{eq:Qfunc_util}
\frac{x}{1+x^2}\phi(x)<Q(x),
\end{equation}
where $\phi(x)$ denotes the pdf of a standard normal distribution. Therefore,
\[
d(\gamma)\leq {\sqrt{2\pi }(1+x^2)\over x}{\rm e}^{x^2\over 2},
\]
where $x=\sqrt{{(1-{1 \over p})\gamma^2 \over (c^2-{\gamma^2\over p})}(1+\tau)}$. The desired result follows   by noting that for $p$ large enough $\gamma^2 \over p$ is negligible, and by assumption, $\gamma\leq c(\ln{\ln{p}})^{1-\eta\over 2}$, where $\eta>0$.
\end{proof}
In a similar manner Corollary \ref{lema:runtime_sphr} captures the tradeoff between the number of projections and the resulting 1-dimensional separability for $\gamma=(\ln{\ln{p}})^{1-\eta\over 2}$ with $d=o(\ln{p})$  projections. This result provides a substantially higher running-time but for a tradeoff in the accuracy. The proof follows similarly to the proof of Corollary (\ref{lema:runtime_sphr_olog}).
\begin{corollary}
\label{lema:runtime_sphr}
 Consider two spherical Gaussian distributions in $\mathds{R}^p$, with means $\mv_1,\mv_2\in\mathds{R}^p$ and covariance matrices  $\Sigma_1=\Sigma_2=\sigma^2 I_p$. Let  $d(\gamma)$ denote  the expected number of projections required  to achieve $\gamma$-separability. If $\gamma$ is such that  $\frac{\gamma}{ c}\leq ({\ln{p}})^{1-\eta\over 2}$, where $\eta>0$ is a free parameter,  then $d(\gamma)= o(p)$.
\end{corollary}
To exemplify the tradeoff implications, consider $\frac{\gamma}{c}=\sqrt{\ln{\ln{p}}}=1.49$, $p=10^4$, and $c=1$. According to %Eq. (\ref{eq:opt_clust_err})
an optimal Bayes classifier this yields 5\% clustering error in 1-dimension. To achieve that error $d(\gamma)\leq 9.24$ projections are sufficient to be examined, on average. On the other hand, for  $\frac{\gamma}{c}=\sqrt{\ln{p}}=3.03$ the clustering error is essentially 0, however, the average number of projections required to achieve this error rate is $d(\gamma)\leq 10^4$.

 The conditions provided in corollary \ref{lema:runtime_sphr_olog} address the similarity between $\gamma$ and $c$ and enable us to construct novel and efficient algorithms employing remarkably small number of projections if $\gamma$ is close to $c$ up to a log-logarithmic factor in $p$.

%-------------------------------------------------
\subsection{The case of $k$-GMM ($k>2$)}
We extend Theorem \ref{thm:1} to the case of $k$ Gaussians.

\begin{lemma}\label{lemma:k-GMM}
Consider $\mv_1,...,\mv_k\in\mathds{R}^p$ and $\sigma_1,...,\sigma_k\in\mathds{R}^+$. Assume that $\Av=(A_1,\ldots,A_p)$ are generated $i.i.d.$~according to $\Nc(0,1)$. Given $\gamma_{\min}>0$, and $i,j\in\{1,\ldots,k\}$ let
\[
c_{(i,j)}={\|\mv_i-\mv_j\|\over  \sqrt{p} (\sigma_i+\sigma_j)}.
\]
Let $c_{\min}\triangleq \min_{i,j}c_{(i,j)}$.  Define  event $\Bc$ as having separability larger than $\gamma_{\min}$ by all pairs of projected Gaussians. Thats is,
\begin{align}
\Bc\triangleq\left\{\left|\left\langle\mv_i-\mv_j,{\Av\over \|\Av\|}\right\rangle\right|\geq  \gamma_{\min} (\sigma_i+\sigma_j):\forall (i,j)\in \{1,\ldots,k\}^2, i\neq j\right\}.\label{eq:def-event-B}
\end{align}
Then,
\begin{equation}\label{eq:kgmm_prob}
\P(\Bc^c)\leq {k^2\over 2}\left(1-2Q \left({\gamma_{\min}\over c_{\min}}\sqrt{{1.1\over 1-{\gamma^2_{\min}\over c^2_{\min} p}}}\;\right)(1- {\rm e}^{-0.002 p})\right).
\end{equation}
\end{lemma}

%To understand the implications of Lemma \ref{lemma:k-GMM},  assume that $p$ is large enough such that the upper bound in \eqref{eq:kgmm_prob} can be well-approximated by
%\begin{align}
%{k^2\over 2}\left(1-2Q\left({\gamma_{\min}\over c_{\min}}\sqrt{1.1}\right)\right).
%\end{align}
%For this bound to be a non-trivial bound, $\left(1-2Q\left({\gamma_{\min}\over c_{\min}}\sqrt{1.1}\right)\right)$ should be smaller than $1/k^2$. This suggests that the argument of the $Q$ function should be close to zero. For for small values of $x$, where $x>0$,
%\begin{align}
%\label{approx:1-2Q}
%1-2Q(x)&={1\over \sqrt{2\pi}} \int_{-x}^x{\rm e}^{-{u^2\over 2}}du\nonumber\approx {1\over \sqrt{2\pi}}  \int_{-x}^x(1-{u^2\over 2})du\nonumber=\sqrt{2 \over \pi} \left( x-{x^3\over 6}\right).
%\end{align}
%Using this approximation for large values of $p$, the upper bound can be approximated as
%\[
%{k^2\over 2}\sqrt{2.2\over \pi} {\gamma_{\min}\over c_{\min}}.
%\]
%Therefore, roughly speaking, in order for the bound to guarantee a small number of projections, $c_{\min}$ should be  $O(k^2 \gamma_{\min})$.
%%\textcolor{red}{
A better understanding of the running time dependency on the number of components $k$ can be derived in the following Corollary and its proof:
\begin{corollary}
\label{cor:num_proj_kGMM}
Consider a mixture of $k$ Gaussian distributions in $\mathds{R}^p$, where component $i$, $i=1,\ldots,k$, is distributed as $\Nc(\mv_i,\sigma_i^2I_p)$.  For $i,j\in\{1,\ldots,k\}$, $i\neq j$, let $c_{\min}\triangleq \min_{i,j}c_{(i,j)}$, where
\[
c_{(i,j)}={\|\mv_i-\mv_j\|\over  \sqrt{p} (\sigma_i+\sigma_j)}.
\]
At  each projection step, assume that all Gaussians are projected using an independently drawn vector $\Av\in\mathds{R}^p$, where $A_1,\ldots,A_p$ are $i.i.d.$~$\Nc(0,1)$.  Let $d(\gamma_{\min},p)$ denote the expected number of projections required to obtain separability $\gamma_{\min}$ between each pair of projected Gaussians. Then, if
\[
\gamma_{\min}\leq (1-\alpha)\sqrt{2\pi \over 1.1} {c_{\min}\over k^2} ,
\]
for some $\alpha \in(0,1)$, then $\limsup_{p\to\infty} d(\gamma_{\min},p) \leq {1\over \alpha}.$
\end{corollary}
\begin{proof}
Consider event $\Bc$ defined in \eqref{eq:def-event-B}, which denotes the desired event where each pair of projected Gaussians satisfy the desired separability. But,
\begin{align}
d(\gamma_{\min},p)&={1\over\P(\Bc)},
\end{align}
where $\P(\Bc^c)$ is upper-bounded by Lemma \ref{lemma:k-GMM}. Taking the limit as $p$ grows to infinity, it follows that
\begin{align}
\limsup_{p}d(\gamma_{\min},p)&\leq {1\over 1- {k^2\over 2}\left(1-2Q\left({\gamma_{\min}\sqrt{1.1}\over c_{\min}}\right)\right)}\label{eq:limsup-d-gamma}
\end{align}
On the other hand, for $x>0$,
\begin{align}
1-2Q(x)&={1\over \sqrt{2\pi}} \int_{-x}^x{\rm e}^{-{u^2\over 2}}du \leq \sqrt{2\over \pi} x.\label{eq:tangent-1-2Q}
\end{align}
Combining \eqref{eq:tangent-1-2Q} and \eqref{eq:limsup-d-gamma}, it follows that
\begin{align}
\limsup_{p}d(\gamma_{\min},p)&\leq {1\over 1- {k^2} \left(\sqrt{1.1\over 2\pi}{\gamma_{\min}\over c_{\min}} \right)}\nonumber\leq  {1\over 1- {k^2} \sqrt{1.1\over 2\pi}{(1-\alpha)\sqrt{2\pi \over 1.1} {1 \over k^2}} }=\frac{1}{\alpha},
\end{align}
where the last inequality follows from our assumption about $\gamma_{\min}$.
\end{proof}

\subsection{Mixture of two arbitrary Gaussians}
\label{subsec:general_GMM}
In this section, we generalize the results of the previous section to  arbitrary Gaussians with covariance matrices $\Sigma_1$ and $\Sigma_2$. But first we note that the notion of separability prescribed by (\ref{eq:sep_sec}) (or any other distance Euclidean-based metric) in the arbitrary mixture setting is a rather conservative one. In particular, since for non-spherical Gaussians the directions of maximal variance corresponding to $\lambda_{\max}$ of each of the covariances do not necessarily align with each other in many realistic high dimensional cases which renders $\lambda_{max}$ as a rather crude spherical estimate. To this end,  better separability prevails between the two Gaussians than what is prescribed by (\ref{eq:sep_sec}), which, in turn, provides lower clustering error.

Conditioned on $\Av=\av$, projecting points $\Xv$ drawn from Gaussian distribution $\Nc(\mv_i,\Sigma_i)$ as $\Xv^T\av$ are distributed as a Gaussian distribution with mean $\E[\langle \Xv,\av\rangle]=\langle\mv_i,\av\rangle,$
and variance $\var(\langle \Xv,\av\rangle)=\av^T\Sigma_i\av.$
%Similarly, conditioned on $\Av=\av$, every point in $\Yc_p$ is distributed as $\Nc(\langle\mv_2,\av\rangle, \av^T\Sigma_2\av)$.
As argued before, the two projected clusters  are  separable, if $|\langle\mv_1,\av\rangle-\langle\mv_2,\av\rangle|>\gamma( \sqrt{\av^T\Sigma_1\av}+\sqrt{\av^T\Sigma_2\av}),$
or
\begin{align}
|\langle\mv_1-\mv_2,\av\rangle|>\gamma\Big( \sqrt{\av^T\Sigma_1\av}+\sqrt{\av^T\Sigma_2\av}\Big),\label{eq:sep-condition-memory}
\end{align}
for some appropriate $\g>0$.
Unlike the condition stated in \eqref{eq:sep-condition-iid}, both sides of \eqref{eq:sep-condition-memory} depend on the direction of $\av$. Therefore, analyzing the following probability
\begin{align}
\P\Big(|\langle\mv_1-\mv_2,\Av\rangle|>\gamma (\sqrt{\Av^T\Sigma_1\Av}+ \sqrt{\Av^T\Sigma_2\Av}\;)\Big),\label{eq:prob-sep-condition-memory}
\end{align}
is more complicated. The following theorems \ref{thm:2} and \ref{thm:2-r} provide lower bounds on this probability for the cases of $\Sigma_1+\Sigma_2$ having a full rank $r=p$, and for the case of partial rank $r< p$, respectively.
\begin{theorem}\label{thm:2}
Consider $\mv_1,\mv_2\in\mathds{R}^p$ and semi-positive definite matrices $\Sigma_1$ and $\Sigma_2$.  Assume that the entries of  $\Av=(A_1,\ldots,A_p)$ are generated $i.i.d.$~according to $\Nc(0,1)$. Let $\lambda_{\max}$ denote the  maximum eigenvalue of $\Sigma_1+\Sigma_2$. Also, given $\gamma>0$, let
\[
\beta \triangleq {2\gamma^2\lambda_{\max}p \over \|\mv_1-\mv_2\|^2}.
\]
Then, for any $\tau>0$, the probability that the 1-dimensional projected Gaussians using a uniformly random direction  are $\gamma$-separated, \ie $\P\Big(|\langle\mv_1-\mv_2,\Av\rangle|\geq  \gamma (\sqrt{\Av^T\Sigma_1\Av}+ \sqrt{\Av^T\Sigma_2\Av}\;)\Big)$,  can be lower-bounded by
\begin{align}
\label{eq:prob_nonsphr_frnk}
\P\Big({A_1^2} >\beta {(1-{1 \over p})\over (1-{\beta\over p})}(1+\tau)\Big)(1- {\rm e}^{-{p-1\over 2}(\tau-\log(1+\tau))}).
\end{align}
\end{theorem}

%---------------------------------------------
In the next theorem, we consider the case where the covariance matrices are not full-rank. Theorem 10 below shows how the expected  number of required projections can dramatically decrease, if the rank of $\Sigma_1+\Sigma_2$ is much smaller than $p$.
\begin{theorem}\label{thm:2-r}
Consider $\mv_1,\mv_2\in\mathds{R}^p$ and semi-positive definite matrices $\Sigma_1$ and $\Sigma_2$.  Assume that the entries of  $\Av=(A_1,\ldots,A_p)$ are generated $i.i.d.$~according to $\Nc(0,1)$. Let $r$ and $\lambda_{\max}$ denote the rank and the maximum eigenvalue of $\Sigma_1+\Sigma_2$, respectively. Also, given $\gamma>0$, $\tau_1\in(0,1)$ and $\tau_2>0$, let
\[
\beta \triangleq {2(1+\tau_2)\gamma^2\lambda_{\max}r \over (1-\tau_1)\|\mv_1-\mv_2\|^2 }.
\]
Then, for any $\tau>0$, the probability that the 1-dimensional projected Gaussians using a uniformly random direction  are $\gamma$-separated, \ie $\P\Big(|\langle\mv_1-\mv_2,\Av\rangle|\geq  \gamma (\sqrt{\Av^T\Sigma_1\Av}+ \sqrt{\Av^T\Sigma_2\Av}\;)\Big)$,  can be lower-bounded by
\begin{align*}
\P\Big({A_1^2} >\beta {(1-{1 \over p})\over (1-{\beta\over p})}(1+\tau)\Big)(1- {\rm e}^{-{p-1\over 2}(\tau-\log(1+\tau))})- {\rm e}^{{p\over 2}(\tau_1+\log(1-\tau_1))}
- {\rm e}^{-{r\over 2}(\tau_2-\log(1+\tau_2))}.
\end{align*}
\end{theorem}
%---------------------------------------------
Similarly to the spherical case, the number of projections required for attaining a separability $\gamma$ can be derived in the following Corollaries \ref{cor:num_proj4beta} and \ref{lema:runtime_nonsphr_olog}. The proofs follow closely the proofs of Corollaries \ref{cor:num_proj4gamma} and \ref{lema:runtime_sphr_olog}.
%To this end, the estimated average number of projection to achieve 1-dimensional severability $\gamma$ can be derived in the following Corollary \ref{cor:num_proj4beta}:
\begin{corollary}\label{cor:num_proj4beta}
Consider two $c$-separable Gaussian distributions in $\mathds{R}^p$ with means $\mv_1,\mv_2\in\mathds{R}^p$ and covariance  matrices $\Sigma_1$ and $\Sigma_2$. Let $\beta \triangleq {2\gamma^2\lambda_{\max}p \over \|\mv_1-\mv_2\|^2}$, where $\lambda_{max}$ denotes the maximal eigenvalue of the matrix $\Sigma_1+\Sigma_2$.
Denote $d(\gamma)$ as the average number of 1-dimensional random projections required to attain $\gamma$-separation in 1-dimension. Then as $p\rightarrow \infty$ we obtain $d(\gamma)\leq \frac{1}{2Q(\sqrt{\beta)}}$.
\end{corollary}

%\begin{corollary}
%\label{lema:runtime_nonsphr}
% Consider two Gaussian distributions in $\mathds{R}^p$, with means $\mv_1,\mv_2\in\mathds{R}^p$ and covariance matrices  $\Sigma_1,\;\Sigma_2$. Let  $d(\gamma)$ denote  the expected number of projections required  to achieve $\gamma$-separability. If $\gamma$ is such that  $\sqrt{\beta} \leq ({\ln{p}})^{1-\eta\over 2}$, where $\eta>0$ is a free parameter,  then $d(\gamma)= o(p)$.
%\end{corollary}
\begin{corollary}
\label{lema:runtime_nonsphr_olog}
 Consider two spherical Gaussian distributions in $\mathds{R}^p$, with means $\mv_1,\mv_2\in\mathds{R}^p$ and covariance matrices  $\Sigma_1,\;\Sigma_2$.  Let  $d(\gamma)$ denote  the expected number of projections required  to achieve $\gamma$-separability. If $\gamma$ is such that  $\sqrt{\beta}=(\ln{\ln{p}})^{1-\eta\over 2}$, where $\eta>0$ is a free parameter,  then $d(\gamma)= o(\ln{p})$.
\end{corollary}

%--------------------------------%-------------------------------------------%--------------------------------------%

\section{Generalization to Arbitrary Distributions}
\label{sec:arbitrary_dist}
It turns out that the Gaussian distribution plays no role in our analysis of the probability to achieve 1-dimensional separability $\gamma$. To see this point, consider data in $\mathds{R}^p$ that is generated according to a mixture of   two (general) distributions, such that under   distribution $i$, $i=1,2$, the data has mean $\mv_i$ and covariance matrix $\Sigma_i$. Further assume that the two data clouds  are $c$-separable, i.e., \eqref{eq:sep_sec} holds. Each distribution in $\mathds{R}^p$ after projection under a random vector $\Av$, where $A_1,\ldots,A_p$ are i.i.d.  $\Nc(0,1)$, maps into a  1-dimensional distribution.

For $\Av=\av$, it is straightforward to check that  the distribution $i$  will be   projected into a distribution that has mean $\langle \mv_i,\av\rangle$ and variance $\av^T \Sigma_i\av$.%
%\begin{align}
%\E[\av^T\Xv]=\av^T\mv_i.
%\end{align}
%Moreover,
%\begin{align}
%{\rm var}(\av^T\Xv)&=\E[(\av^T\Xv-\av^T\mv_i)^2]=\av^T\mv_i\nonumber=\E[\av^T(\Xv-\mv_i)(\Xv-\mv_i)^T\av]\nonumber\\&=\av^T\E[(\Xv-\mv_i)(\Xv-\mv_i)^T]\av\nonumber\\
%&=\av^T\Sigma_i \av\nonumber.\\
%\end{align}
Therefore,   the two projected clouds  are  $\gamma$-separable after projection:
\begin{align}
|\langle\mv_1-\mv_2,\av\rangle|>\gamma\Big( \sqrt{\av^T\Sigma_1\av}+\sqrt{\av^T\Sigma_2\av}\Big).
\end{align}
The rest of our analysis in Sections \ref{subsec:spherical_GMM} and \ref{subsec:general_GMM} focuses on the probability of these events for the cases of $\Sigma_i=\sigma_i^2I$, and general $\Sigma_i$, respectively. Note that our analysis rely only on the randomness in the generation of  the projection vector $\Av$ and not on the data itself nor its being Gaussian. Hence, Corollaries \ref{lema:runtime_sphr_olog} and \ref{lema:runtime_nonsphr_olog} bounding on the  number of projections readily generalize to arbitrary distributions in  $\mathds{R}^p$.

The remaining question is how does $\gamma$ relate to the clustering error for non-Gaussian projected data. For the Gaussian case, we have shown that the desired separability $\gamma$ can be computed from the prescribed error $e$ as $\gamma \geq Q^{-1}(e)$ (Eq. \ref{eq:opt_clust_err}). One path to tackle the non-Guassian case is by observing that for many general distributions the central limit theorem for large $p$ reveals that the projected data from a non-Gaussian arbitrary distribution converges to a Gaussian one, which allows us to use the $Q$-function as suggested above. Such concentration results for random i.i.d Gaussian projections have been studied in \cite{Dasgupta_concntrt} for arbitrary distributions with finite second moments, and can be utilized here as well. We provide details in the full version.

% What changes is the resulting distribution in 1-dimension, a general distribution $G(x)$  instead of $Q(x)$ for the tail of the distribution. %After finding two clusters of points  in $\mathds{R}$ that are $\gamma$-separable, the question is what probability of error does this correspond to?

\section{Algorithm}
\label{sec:alg}
Inspired by Theorems \ref{thm:1}, \ref{thm:2}, \ref{thm:2-r}, in this section, we propose Algorithm \ref{fig_algflow}  for efficiently  clustering a mixture of Gaussian distributions in high-dimensional space $\mathds{R}^p$, and its extension for spherical Gaussians mixture. The algorithm receives as input the $n\times p$ data matrix $X$, a prescribed  error - $e$, and a maximum  number of 1-dimensional projections - $M$. $M$ can be estimated, for example, as $o(\ln{p})$ before execution based on corollaries \ref{lema:runtime_sphr_olog}, or \ref{lema:runtime_nonsphr_olog}, by assuming that the prescribed error $e$ can be achieved efficiently. The algorithm sequentially performs 1-dimensional projections, where each projection's direction is chosen uniformly at random. After each random projection, the parameters of the projected mixture of Gaussians in 1-dimension and its corresponding clustering error are estimated. This process is iterated until either the desired accuracy $e$ is achieved by the current projection, or the maximum number of  projections $M$ is reached.

For the spherical case, one can use Lemma \ref{lemma:5} to estimate $c$ on-the-fly via $\bar{c}$ as
\begin{equation}\label{eq:capprox}
\hat{c}_i=\sqrt{\frac{1}{i}\sum_{j=1}^i \hat{\gamma}_j^2},
\end{equation}
where $\hat{\gamma}_j$ is the estimated 1-dimensional separability from projection $\boldsymbol{A^j}$. We can now readily derive the number of projections to achieve $\gamma=Q^{-1}(e)$ via Theorem \ref{thm:1}.

To this end, we can distinguish between the two cases in our algorithm execution: the case of an arbitrary mixture (Alg. \ref{fig_algflow}), for which it is unknown if the GMM is comprised of spherical Gaussians, hence $M$ can be bounded by $o(\ln{p})$. Or, the case of known spherical mixture, where Theorem \ref{thm:1} is utilized to estimate the expected  number of projections to achieve $\gamma$ based on Lemma \ref{lemma:5} by using equation (\ref{eq:capprox}).

%The last case, in particular, allows on-the-fly estimation of the expected number of projections required to achieve $e$ by employing Theorem \ref{thm:1} and its corresponding Corollary \ref{lema:runtime_sphr_olog}.

We provide numerical experiments using our algorithm in Appendix \ref{sec:experiments}, where we validate the bounds provided by Theorems \ref{thm:1}, \ref{thm:2}, and \ref{thm:2-r}, and the derived Corollaries.

\begin{algorithm}
    \SetAlgoLined
    \KwData{$X - n\times p$ data matrix, $e$ - error, $M$ - projection budget}
    \KwResult{$\Cc^*$ - decision boundary}
    Initialization: $i=1,\hat{e}=\infty$\\
    \While {$i<M$}{
       Project to random direction: $ X\boldsymbol{A^i}$\\
       Learn 1-dimensional parameters: $(\hat{m}^i_1,\hat{m}^i_2,\hat{\sigma}^i_1,\hat{\sigma}^i_2,\hat{w}^i_1)$\\
       Learn a separator $\Cc^*$ and compute $\hat{e}$\\
       \If {$\hat{e}<e$}{
           return($\Cc^*$)
       }
     }
    print("Error not Achievable")\\
    return($\Cc^*$)
\caption{ClusterGMM}
\label{fig_algflow}
\end{algorithm}
\bibliographystyle{abbrv}

\appendix

%\section{My Proof of Theorem 1}

%--------------------------------%--------------------------------
%--------------------------------%--------------------------------
\section{Sample complexity}
\label{sec:smpl_complex}
In this section we study the sample complexity of our algorithm. The study is done at the 1-dimensional setting where the clustering is performed using Algorithm 3.3 of \cite{price} to estimate the parameters of the projected mixture of two Gaussians. Algorithm 3.3 is a variation of  the well-known method of moments algorithm proposed by Pearson in \cite{pearson1894contributions}.  The following result from  \cite{price} summarizes the performance of Algorithm 3.3 in estimating the parameters of a mixture of two general Gaussians in 1-dimension.
\begin{theorem}[Theorem 3.10 in \cite{price}]\label{thm:thm-par-estimation}
Consider a mixture of two Gaussian distribution   $w \Nc(\mu_1,\sigma_1)+(1-w)\Nc(\mu_2,\sigma_2)$. Let $\sigma^2=w(1-w)(\mu_1-\mu_2)^2+w\sigma_1^2+(1-w)\sigma_2^2$ denote the variance of this distribution. Then, given $n=O({1\over \epsilon^2} \log{1\over \delta})$ samples, Algorithm 3.3, with probability $1-\delta$, returns estimates of the parameters as $(\hat{\mu}_1,\hat{\mu}_2,\hat{\sigma}_1,\hat{\sigma}_1,\hat{w})$, which under the right permutation of the indices, satisfy the following guarantees, for $i=1,2$,
\begin{itemize}
\item If $n\geq \left({\sigma^2\over |\mu_1-\mu_2|^2}\right)^6$, then $|\mu_i-\muh_i|\leq \epsilon|\mu_1-\mu_2|$, $|\sigma_i^2-\hat{\sigma}_i^2|\leq \epsilon|\mu_1-\mu_2|^2$, and $|w-\hat{w}|\leq \epsilon$.
\item If $n\geq \left({\sigma^2\over |\sigma^2_1-\sigma^2_2|}\right)^6$, then  $|\sigma_i^2-\hat{\sigma}_i^2|\leq \epsilon|\sigma_1^2-\sigma_2^2|+|\mu_1-\mu_2|^2$, and $|w-\hat{w}|\leq \epsilon+{|\mu_1-\mu_2|^2\over |\sigma^2_1-\sigma_2^2|}$.
\item For any $n\geq 1$, the algorithm performs as well as assuming the mixture is a single Gaussian, and  $|\mu_i-\muh_i|\leq |\mu_1-\mu_2|+\epsilon\sigma$, and $|\sigma_i^2-\hat{\sigma}_i^2|\leq |\mu_1-\mu_2|^2+|\sigma_1^2-\sigma_2^2|+\epsilon\sigma^2$.
\end{itemize}
\end{theorem}
The following corollary is a direct result of Theorem \ref{thm:thm-par-estimation}. It shows that, if the two components of a Gaussian mixture model are separated enough in 1-dimension, given sufficient  sample, Algorithm 3.3 of \cite{price} returns accurate estimates of \emph{all} parameters.
\begin{corollary}\label{cor:5}
Let $(X_1,\ldots,X_n)$ denote $n$ i.i.d.~samples of a mixture of two $c$-separable Gaussians  $w \Nc(\mu_1,\sigma_1)+(1-w)\Nc(\mu_2,\sigma_2)$, where  $\mu_1<\mu_2$ and $\sigma_1=\sigma_2$. Further assume that the separability $c=|\mu_1-\mu_2|/(\sigma_1+\sigma_2)$ in 1-dimension is larger than $\gamma_{\min}$. Let $(\hat{\mu}_1,\hat{\mu}_2,\hat{\sigma}_1,\hat{\sigma}_2,\hat{w})$ denote the estimates of $(\mu_1,\mu_2,\sigma_1,\sigma_2,w)$ returned by Algorithm 3.3 of \cite{price}. Then, if  $n=O({1\over \epsilon^2}\log{1\over \delta})$ and $n\geq {1\over (2\gamma_{\min})^{12}}$, then  $|\mu_i-\muh_i|\leq \epsilon|\mu_1-\mu_2|$, $|\sigma_i^2-\hat{\sigma}_i^2|\leq \epsilon|\mu_1-\mu_2|^2$, and $|w-\hat{w}|\leq \epsilon$.
\end{corollary}

Note that as $\gamma_{\min}$ converges to zero, the required number of samples for accurate estimation of the parameters grows to infinity. On the other hand, too small separability $\gamma$ corresponds to  large overlap of the two Gaussians. Hence, as confirmed in the following lemma,  unless the weights of the two Gaussians are very non-uniform, i.e.~$\min(w,1-w)$ is far from $0.5$, small separability $\gamma$, corresponds to  poor clustering performance.
%????????????%????????????%????????????%????????????%????????????%????????????%????????????
%To Further enhance the applicability of Theorem \ref{thm:thm-par-estimation}, we demonstrate in the following lemma that for different regimes of the mixture coefficients  the error bound attained by the optimal Bayesian classifier divides into two extreme cases - detecting a mixture of 2 Gaussians vs. detecting a single Gaussian.
%The following lemma derives a lower bound on clustering error  for a mixture of two $c$-separable Gaussians $w\Nc(\mu_1,\sigma_1)+(1-w)\Nc(\mu_2,\sigma_2)$, with equal variance $\sigma_1=\sigma_2$, in terms of $c$ and $w$.
\begin{lemma}\label{lemma:lb-on-e-opt}
Consider i.i.d.~points generated as $w\Nc(\mu_1,\sigma)+(1-w)\Nc(\mu_2,\sigma)$. Without loss of generality, assume that  $\mu_1\leq \mu_2$ and $w<0.5$.  Let $\gamma=(\mu_2-\mu_1)/(2\sigma)$. Also, let $e_{\rm opt}$ denote the error probability of an optimal Bayesian classifier. Then, if $w\leq  0.1$,
\begin{align}
e_{\rm opt}\geq wQ\left(-{1\over \gamma}+\gamma\right).\label{eq:lower-bd-1-e-opt}
\end{align}
For $w\in(0.1,0.5]$,
\begin{align}
e_{\rm opt}\geq wQ\left(\gamma\right).\label{eq:lower-bd-2-e-opt}
\end{align}
%\textcolor{blue}{Also, for
%\[
%w\in\Big({1\over 1+{\rm e}^{2c^2}},0.5\Big],
%\]
%we have
%\[
%e_{\rm opt}\geq Q(2c).
%\]}
\end{lemma}
\begin{proof}
The optimal \emph{Bayesian} classifier, which has access to the parameters  $({\mu}_1,{\mu}_2,{\sigma},{w})$, divides the real line at
\begin{align}
t_{\rm opt}={\mu_1+\mu_2\over 2}-{\sigma^2\over (\mu_1-\mu_2)}\ln {w\over 1-w},\label{eq:def-t-opt}
\end{align}
and achieves a classification error equal to
\begin{align}
e_{\rm opt}&=w\P(\mu_1+\sigma_1Z\geq t_{\rm opt})+(1-w)\P(\mu_2+\sigma_2Z\leq t_{\rm opt})\nonumber\\
&=wQ\left({\mu_2-\mu_1\over 2\sigma}-{\sigma\over (\mu_1-\mu_2)}\ln {w\over 1-w}\right)+(1-w)Q\left({\mu_2-\mu_1\over 2\sigma}+{\sigma\over (\mu_1-\mu_2)}\ln {w\over 1-w}\right)\nonumber\\
&=wQ\left(\gamma+{1\over 2\gamma}\ln {w\over 1-w}\right)+(1-w)Q\left(\gamma-{1\over 2\gamma}\ln {w\over 1-w}\right),\label{eq:e-opt-char}
\end{align}
where $Z\sim\Nc(0,1)$. Note that since by assumption $w<1-w$, $\ln {w\over 1-w}\leq 0$.  Therefore,
\[
Q\left(\gamma-{1\over 2\gamma}\ln {w\over 1-w}\right)\leq Q\left(\gamma+{1\over 2\gamma}\ln {w\over 1-w}\right).
\]
Keeping the larger $Q$ term, it follows from \eqref{eq:e-opt-char} that
\begin{align}
e_{\rm opt}&\geq wQ\left(-{1\over 2\gamma}\ln {1-w\over w}+\gamma\right).
\end{align}
For $w\leq 0.1$,  $0.5\ln {1-w\over w}\geq 0.5\ln {1-0.1 \over 0.1 }>1$. Therefore, since $Q(\cdot)$ is a monotonically decreasing function of its argument, \eqref{eq:lower-bd-1-e-opt} follows. The result for $w\in(0.1,0.5)$ stated in \eqref{eq:lower-bd-1-e-opt} follows  by noting that $-{1\over 2\gamma}\ln {1-w\over w}+\gamma\leq \gamma$.
\end{proof}

Therefore, if the ultimate goal is to achieve a reasonable clustering error through multiple random projections, for those directions with too small separability $\gamma$, we only need to identify  them and discard them. In other words, for such directions, it is not necessary to estimate all the parameters of the projected Gaussians accurately, as they ultimately are not going to be used for clustering.  The following lemma provides a mechanism for identifying and discarding all directions that have a separability smaller than some threshold. It states that given $n=O({1\over \epsilon^2}\log{1\over \delta})$ i.i.d.~samples of two Gaussians with separability $\gamma$,   Algorithm 3.3 of \cite{price} estimates the parameters of the two Gaussians such that the estimated separability  is upper-bounded by $ {3\gamma+\epsilon \over 1-2\sqrt{\gamma^2+\epsilon}}$, with probability larger than $1-\delta$.

\begin{lemma}\label{lemma:estimated-c-bd}
Let $(X_1,\ldots,X_n)$ denote $n$ i.i.d.~samples of a mixture of two $\gamma$-separable Gaussians  $w \Nc(\mu_1,\sigma_1)+(1-w)\Nc(\mu_2,\sigma_2)$, where  $\sigma_1=\sigma_2$, $\gamma=(\mu_2-\mu_1)/(\sigma_1+\sigma_2)<1/2$ and $\mu_1<\mu_2$. Let $(\hat{\mu}_1,\hat{\mu}_2,\hat{\sigma}_1,\hat{\sigma}_2,\hat{w})$ denote the estimates of $(\mu_1,\mu_2,\sigma,\sigma,w)$ returned by Algorithm 3.3 of \cite{price}.  Then, if $n=O({1\over \epsilon^2}\log{1\over \delta})$,  with probability larger than $1-\delta$,
\[
{|\muh_1-\muh_2| \over \hat{\sigma}_1+\hat{\sigma}_2}\leq  {3\gamma+\epsilon \over 1-2\sqrt{\gamma^2+\epsilon}}.
\]
\end{lemma}

% that  shows that if $c$ is smaller than some threshold $c_{m}$, with high probability,   the estimated separability does not exceed a  In such cases for a given $\epsilon$ we could discard 1D projections that yield small approximate separability knowing that the approximated mixture parameters are inaccurate. The following corollary demonstrates the same concept for when the threshold is at $\frac{1}{8}$.

\begin{proof}[Proof of Lemma \ref{lemma:estimated-c-bd}]
By Theorem \ref{thm:thm-par-estimation}, for $n=O(\epsilon^2\log{1\over \delta})$, with probability $1-\delta$, there exists a permutations of indices, such that $|\mu_i-\muh_i|\leq |\mu_1-\mu_2|+\epsilon\sigma$ and $|\sigma_i^2-\hat{\sigma}_i^2|\leq |\mu_1-\mu_2|^2+|\sigma_1^2-\sigma_2^2|+\epsilon\sigma^2= |\mu_1-\mu_2|^2+\epsilon\sigma^2$. Therefore, by the triangle inequality,
\begin{align*}
|\muh_1-\muh_2|\leq \sum_{i=1}^2 |\mu_i-\muh_i| + |\mu_1-\mu_2|\leq 3|\mu_1-\mu_2|+\epsilon \sigma.
\end{align*}
Hence, since $\sigma^2=w(1-w)(\mu_1-\mu_2)^2+\sigma_1^2$,
\begin{align}
{|\muh_1-\muh_2| \over \hat{\sigma}_1+\hat{\sigma}_2}&\leq {3|\mu_1-\mu_2|+\epsilon \sigma\over 2\sigma_1-2\sqrt{|\mu_1-\mu_2|^2+\epsilon\sigma^2}}\nonumber\\
& = {3|\mu_1-\mu_2|+\epsilon \sqrt{w(1-w)(\mu_1-\mu_2)^2+\sigma_1^2}\over 2\sigma_1-2\sqrt{|\mu_1-\mu_2|^2+\epsilon w(1-w)(\mu_1-\mu_2)^2+\epsilon \sigma_1^2}}\nonumber\\
&\stackrel{(a)}{=} {3\gamma+\epsilon \sqrt{w(1-w)\gamma^2+0.25}\over 1-\sqrt{4\gamma^2+4\epsilon w(1-w)\gamma^2+\epsilon}}\nonumber\\
&\stackrel{(b)}{\leq } {3\gamma+0.5\epsilon \sqrt{\gamma^2+1}\over 1-\sqrt{4\gamma^2+\epsilon(1+ \gamma^2)}}\nonumber\\
&\stackrel{(c)}{\leq } {3\gamma+\epsilon \over 1-2\sqrt{\gamma^2+\epsilon}},
\end{align}
where $(a)$ follows by dividing the nominator and denominator by $2\sigma_1$ and $(b)$ holds because $w(1-w)\leq 0.25$. Finally $(c)$ holds, since by assumption $\gamma^2<1$.
\end{proof}

To shed more light on the implications of Lemma \ref{lemma:estimated-c-bd}, consider a mixture of two 1-dimensional Gaussians with equal variance and separability $\gamma$ smaller than  ${1\over 8}$. Then, given  $n=O({1\over \epsilon^2}\log{1\over \delta})$ i.i.d.~samples,  with probability larger than $1-\delta$, the estimated separability (using parameters derived from  Algorithm 3.3 of \cite{price}) is smaller than
\[
{{3\over 8}+\epsilon \over 1-2\sqrt{({1\over 8})^2+\epsilon}}={1\over 2}+o(\epsilon).
\]
Therefore, if after performing each random projection, we  estimate the parameters of the two Gaussians using Algorithm 3.3 of \cite{price} and then estimate the separability of the two Gaussians as ${|\muh_1-\muh_2| \over \hat{\sigma}_1+\hat{\sigma}_2}$ and discard all those directions that have estimated separability smaller than $1\over 2$, we would, with high probability, discard all directions with separability smaller than $1/8$. For directions with separability larger than $1/8$, we need to have enough samples to estimate the parameters accurately. The required number of samples  for achieving this goal is shown in Corollary \ref{cor:5}, which follows directly from Theorem 3.10 of \cite{price}. Note that using this procedure,  directions with estimated separabilities smaller than $0.5$  include those directions with    separabilities  in $({1\over 8},{1\over 2})$, for which, with high probability, we have estimated the parameters accurately, and those directions with separabilities  smaller than ${1\over 8}$, for which we a crude estimate of the parameters.

Finally, for directions for which we accurately estimate  the parameters of the projected Gaussians, the following lemma, connects the  error in estimating the parameters $(\mu_1,\mu_2,\sigma_1,\sigma_2,w)$ to the error in estimating the  clustering  error. Since the ultimate goal of our algorithm is to find a direction which yields a desired clustering error, it is important to establish such a connection, which, given the desired clustering error, characterizes some sufficient  accuracy in estimating the parameters.

%To this end, we have established an apparatus to discard projections in which the mixtures parameters can not be approximated to a good accuracy unless the sample size is extremely large. Next we aim at bounding the clustering error with respect to a given estimated parameters accuracy - $\epsilon$:

%\subsection{$\sigma_1=\sigma_2$}

\begin{lemma}\label{lemma:est-error-to-class-error}
Consider $(X_1,\ldots,X_n)$ that are generated i.i.d.~according to a mixture of two $\gamma$-separable Gaussians  $w \Nc(\mu_1,\sigma_1)+(1-w)\Nc(\mu_2,\sigma_2)$, where  $\sigma_1=\sigma_2$, $w\in[w_{\min},0.5]$, $\mu_1<\mu_2$ and $\gamma\in[\gamma_{\min},\gamma_{\max}]$. Let  $(\hat{\mu}_1,\hat{\mu}_2,\hat{\sigma}_1,\hat{\sigma}_2,\hat{w})$ denote the estimate of the unknown parameters  $({\mu}_1,{\mu}_2,{\sigma}_1,{\sigma}_2,{w})$. Let $e_{\rm opt}$ and $\hat{e}$   denote the minimum achievable classification error and the achieved  clustering error based on the estimated parameters, respectively. Then, if $|\mu_i-\muh_i|\leq \epsilon|\mu_1-\mu_2|$, $|\sigma_i^2-\hat{\sigma}_i^2|\leq \epsilon|\mu_1-\mu_2|^2$,  $|w-\hat{w}|\leq \epsilon$,  and
\[
(16\gamma_{\max}^2+8\gamma_{\max}\ln{1-w_{\min}\over w_{\min}}+2\gamma_{\max}\epsilon)\epsilon<{1\over 2},
\]
we have \footnotesize
\[
|\hat{e}-e_{\rm opt}|\leq   \left( 2\gamma +{1\over w_{\min}\gamma}+\left({1\over \gamma}+2\gamma\right) \ln {1-w_{\min}\over w_{\min}}+{8\gamma_{\max}^2   \over \gamma }+{2\gamma}   \left(  4\gamma+2\ln {1-w_{\min}\over w_{\min}}\right)^2\right) \epsilon+ Q\left({1\over 4\gamma\epsilon }+\epsilon_1\right)+\epsilon_2,
\] \normalsize
where $\epsilon_1=o(1/\epsilon)$ and $\epsilon_2=o(\epsilon)$.
\end{lemma}

\section{Experiments}
\label{sec:experiments}
Our experimental validation focuses on the performance of the algorithm in light of the theoretical bounds on the clustering error and the number of projections required to attain a given error. These empirical studies suggest that the algorithm requires a very small number of projections in order to achieve competitive clustering accuracy for a range of separability values underlying the difficulty of the high dimensional problem. We present our empirical results for the two cases of spherical and non-spherical Gaussians.

\subsection{Experiments with spherical Gaussians}

In the spherical Gaussians case, we examine 3 trade-offs:
\begin{enumerate}
\item accuracy vs.~separability, where we  validate our algorithms error rate by comparing it against the bounds provided by A Bayes classifier that uses the true parameters,
\item number of projections vs.~separability, where we show that the bound provided by Theorem \ref{thm:1} is a tight upper bound on  the expected value of our algorithms' required number of projections, and
\item accuracy vs. number of projections, where we show the fast convergence of our algorithm to the optimal error that can be achieved for a given  separability.
\end{enumerate}

\noindent \textbf{Accuracy vs. separability:} In this experiment we fix the number of projections to 50, and report the  true minimum clustering error achieved after performing these many random projections. We compare this error with the theoretical bound of the Bayes classifier error in the high dimension $Q(\sqrt{p}c)$, and for the 1-dimensional Bayes classifier $Q(c)$. The experiment is repeated for a range of separability values, and  the result is reported in Fig. \ref{fig:acc_vs_sep} for mixture realizations of size 4K data points in $\mathds{R}^3$, $\mathds{R}^{100}$ and $\mathds{R}^{1000}$ with equal probabilities, \ie $w_1=w_2=0.5$. It is observed that the error rates achieved by our algorithm are always better than the error bound in 1-dimension, however, as expected, it does not outperform the optimal Bayes classifier error in the high dimension  which has access to the underlying  parameters. Also, as expected,  the gap between the algorithm's performance and that of an optimal Bayes classifier grows as the dimension $p$ increases.

\begin{figure}[!htpb]
\begin{center}
    \psfig{figure=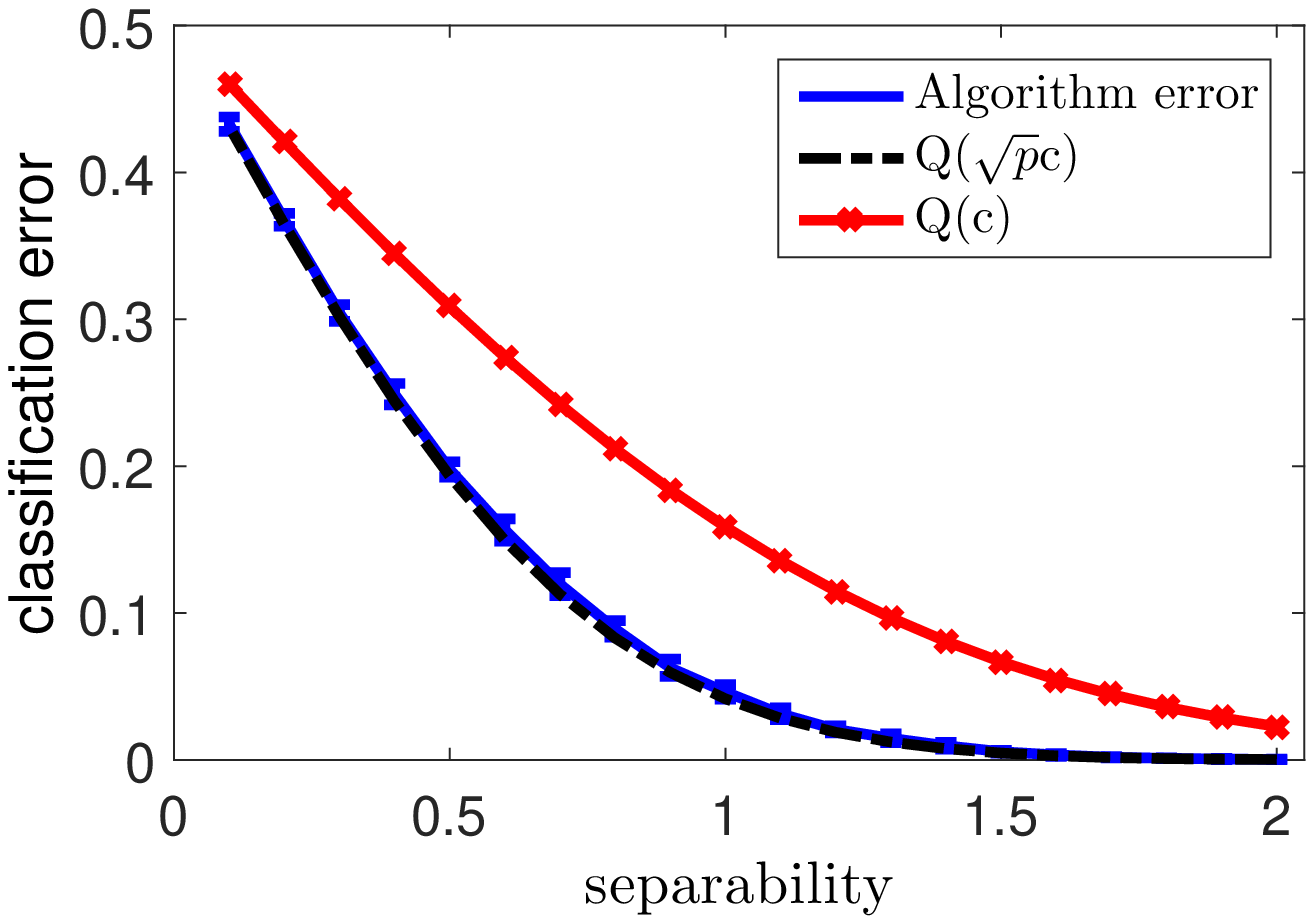,height=50mm,width=50mm}
    \psfig{figure=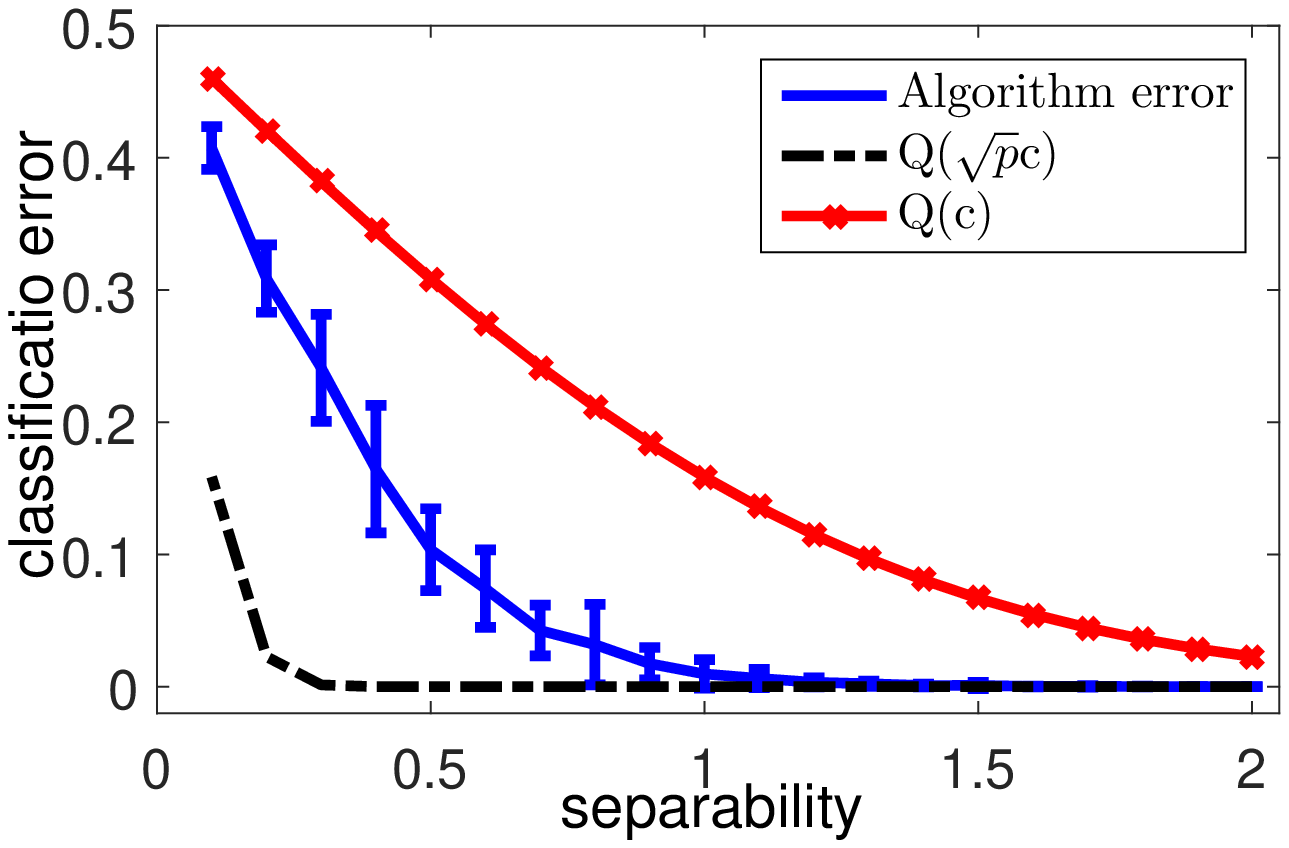,height=50mm,width=50mm}
        \psfig{figure=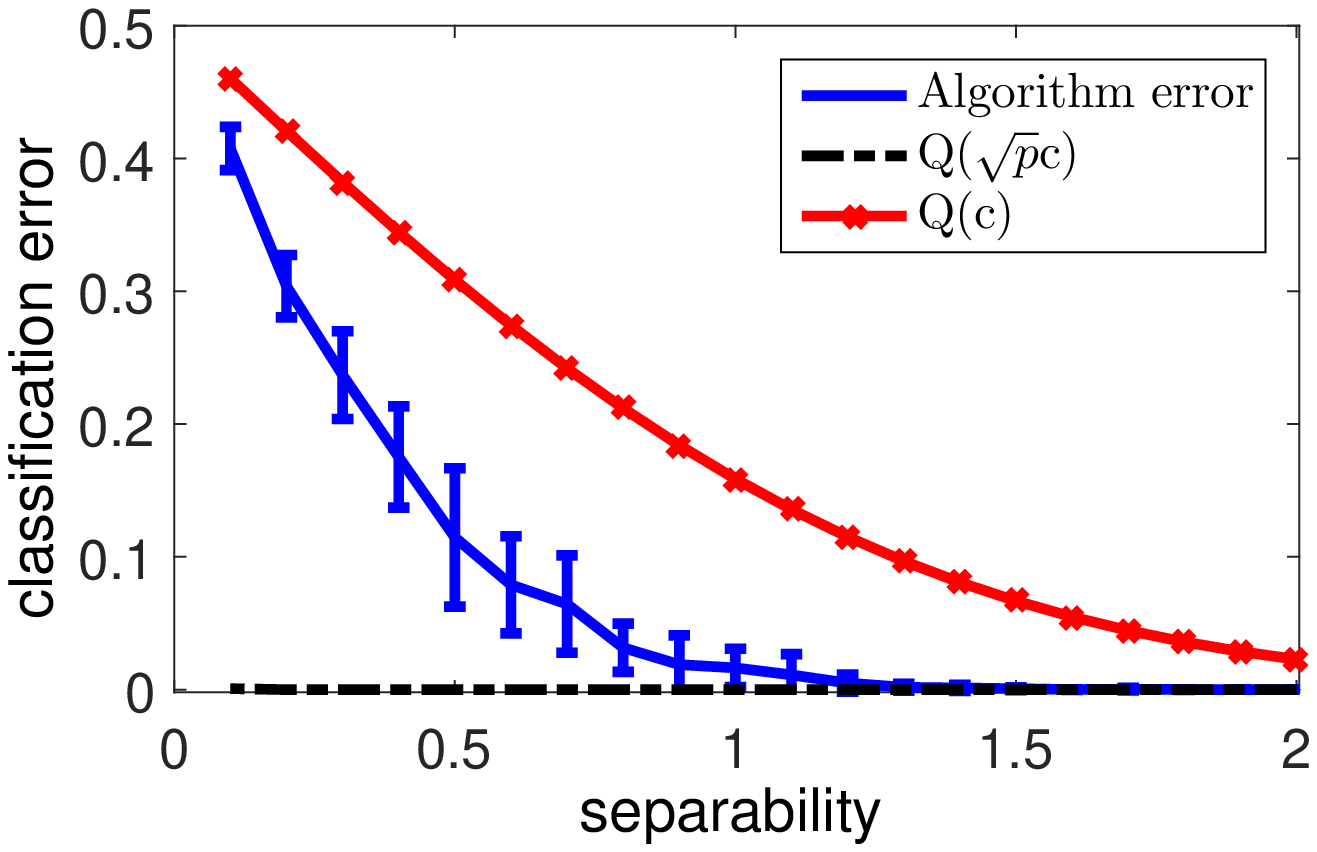,height=50mm,width=50mm}

\caption{Error vs. separability shown for realizations of 10K data points generated by an equal probability mixture of two Gaussians in \textbf{Left}: $\mathds{R}^3$, \textbf{Middle}: $\mathds{R}^{100}$, \textbf{Right}: $\mathds{R}^{1000}$  (note that $Q(\sqrt{p}c)\approx 10^{-7}$)}
 \label{fig:acc_vs_sep}
 \end{center}
 \end{figure}

\noindent \textbf{Number of projections vs. separability:} In this experiment, 10K points  are generated from a mixture of two spherical Gaussian distributions in $\mathds{R}^{100}$, with $w_1=w_2$ and $\sigma_1=\sigma_2$.  The user's desired error is fixed  at $e=20\%$. For different high-dimensional separabilities  $c$, we measure the number of projections used until  both the true error ${e_i}$ and the estimated error  ${\hat{e}_i}$  of one of the random 1-dimensional projections are less than $20\%$. Fig. \ref{fig:proj_vs_acc}-Left plots the number of projections scanned until the prescribed accuracy is attained for various $c$ values. We also plot the lower bound provided by Corollary \ref{cor:prob_csep}, as the inverse of the probability bound defined there  for the given dimension $p$, separability $c$, and $\gamma = Q^{-1}(e)$. It can be observed that the  mean number of projections to achieve the prescribed error is   tightly bounded by the upper bound provided by Corollary \ref{cor:prob_csep}.

\noindent \textbf{Error vs. number of projections:} In this experiment the separability $c$ is fixed at values $0.1,\;0.5,\;1$ and $2$. For each values we increase the number of projections scanned and generate a classification based on the best predicted error for each projection and select the minimal error classification. Fig. \ref{fig:proj_vs_acc}-Right reports the accuracy values for increasing number of projections and for various separability values. A saturation point is observed for each $c$ value at a different location agreeing with the error values reported in Fig. \ref{fig:acc_vs_sep}-middle for the respective separability values therein. The experiment points to the accurate the number of projections that achieves the minimal possible error for our algorithm. As seen the minimal error is approached closely after just a few projections, suggesting the speed and efficiency in which the algorithm can cluster the data to a prescribed error that reasonably corresponds to the high dimensional separability.

\begin{figure}[!htpb]
\begin{center}
    \psfig{figure=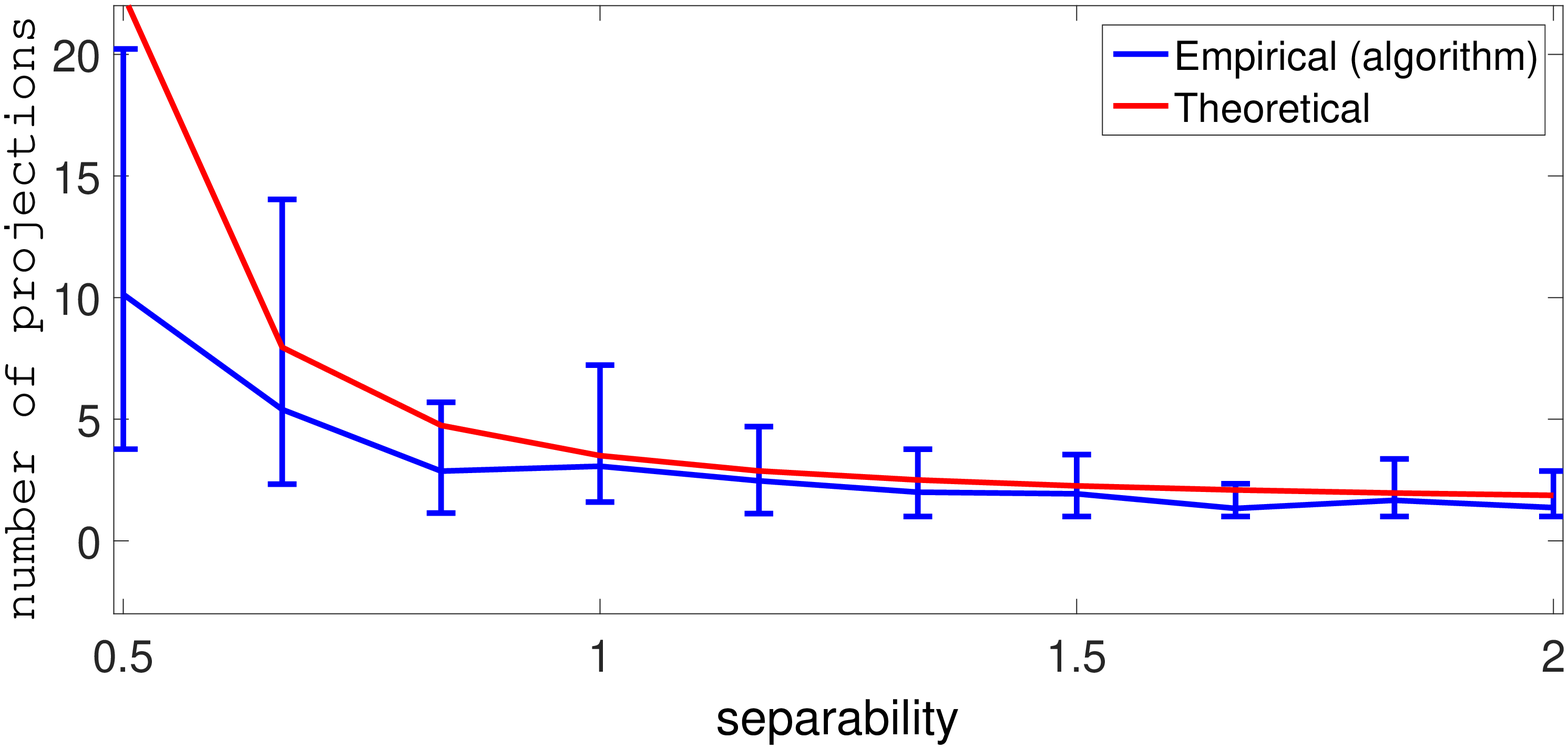,height=62mm,width=67mm}
    \psfig{figure=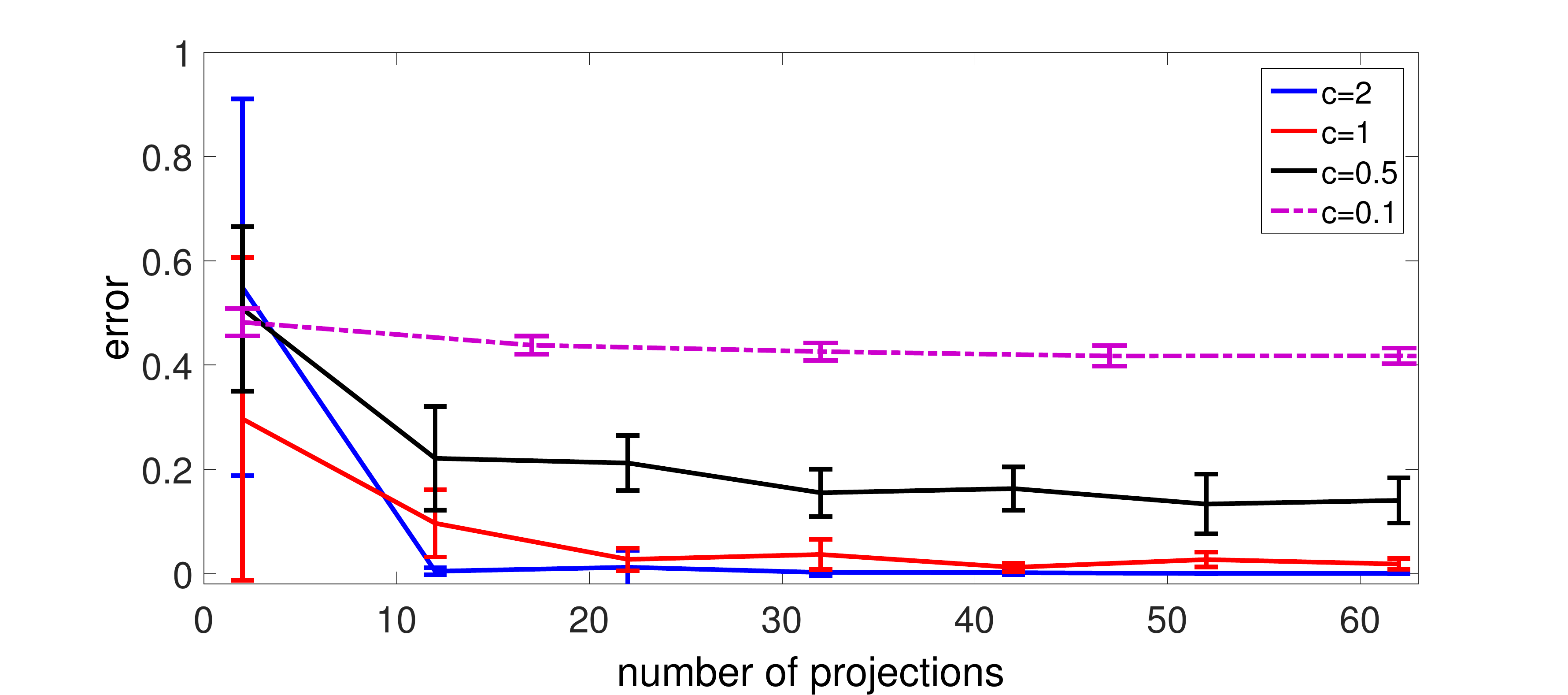,height=62mm,width=67mm}
\caption{Left: Projection order vs. separability to reach 20\% error. Algorithm performance compared with the theoretical upper-bound from Theorem \ref{thm:1}. Right: number of projections vs. accuracy for separability values 0.1, 0.5, 1, and 2. Data contains 10K points realization for a mixture in $\mathds{R}^{100}$.}
 \label{fig:proj_vs_acc}
 \end{center}
 \end{figure}

\subsection{Non-spherical Gaussians}
\noindent \textbf{Accuracy vs. rank:}
In the non-spherical case, we consider a mixture of Gaussians in  $R^{1000}$ and generate 10K points. The respective covariance matrices are generated by setting a fraction $\zeta$ of the dimension to be populated by points of both Gaussians. Then the same fraction $\zeta$  is used to uniformly sample another subspace for each Gaussian to be populated as well. In this experiment, we fix the high-dimensional separability as $c=0.5$, but change $\zeta$ so that the summation of the covariance matrices rank $r=\rank(\Sigma_1+\Sigma_2)$ is changed accordingly.
For each value of $r$,  we generate the same number of random projections and record the best error attained in 1-dimension by exploring 50 projections at most.

%Our empirical validation involves the respective Theorem \ref{thm:2} and the derived corollary \ref{cor:10}.
First, we examine the error as function of the rank of the covariance summation matrix ($\Sigma_1+\Sigma_2$). Fig. \ref{fig:nonSpher}-Left demonstrates this error. We note that as the matrix approaches the full rank (with equal variance in the populated dimensions) the error of our algorithm approaches the error 0.1 attained at the case of spherical Gaussians for $c=0.5$.
%The full regime of error reported here is within the range between the error of the ideal separator in the high dimension - $Q(\sqrt(n)0.5\approx 10^{-7}$) and the 1D separator error $Q(0.5)\approx 0.3$.

\noindent \textbf{Number of projections vs. rank:} Next we examine the number of projections needed to achieve a prescribed error as a function of the rank of the summation matrix for mixtures with the same parameters. We validate the bound provided by Theorem \ref{thm:2-r}. In Fig. \ref{fig:nonSpher}-Right we report results for a 4\% error prescribed.
\begin{figure}[!htpb]
\begin{center}
    \psfig{figure=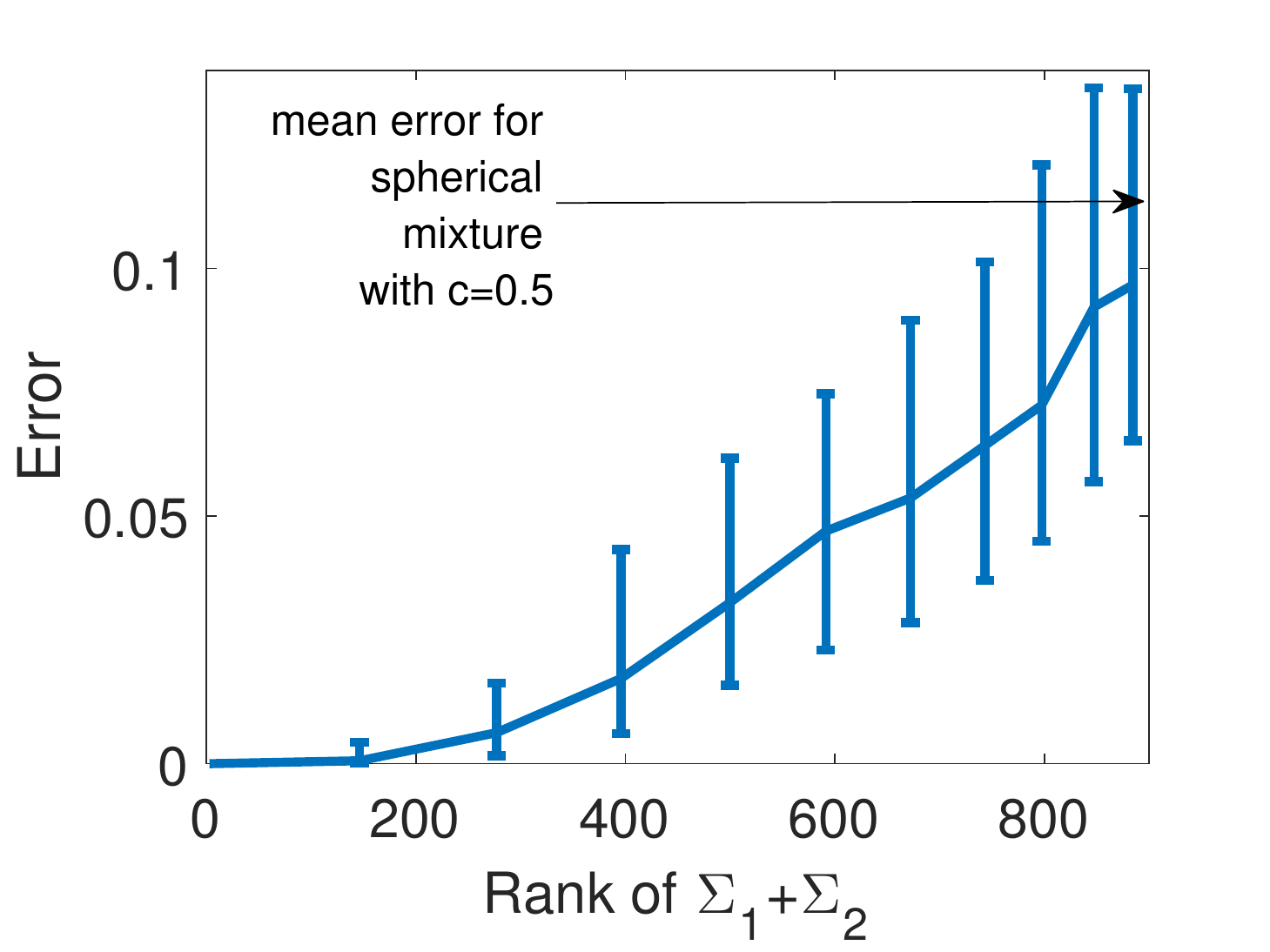,height=65mm,width=61mm}
    \psfig{figure=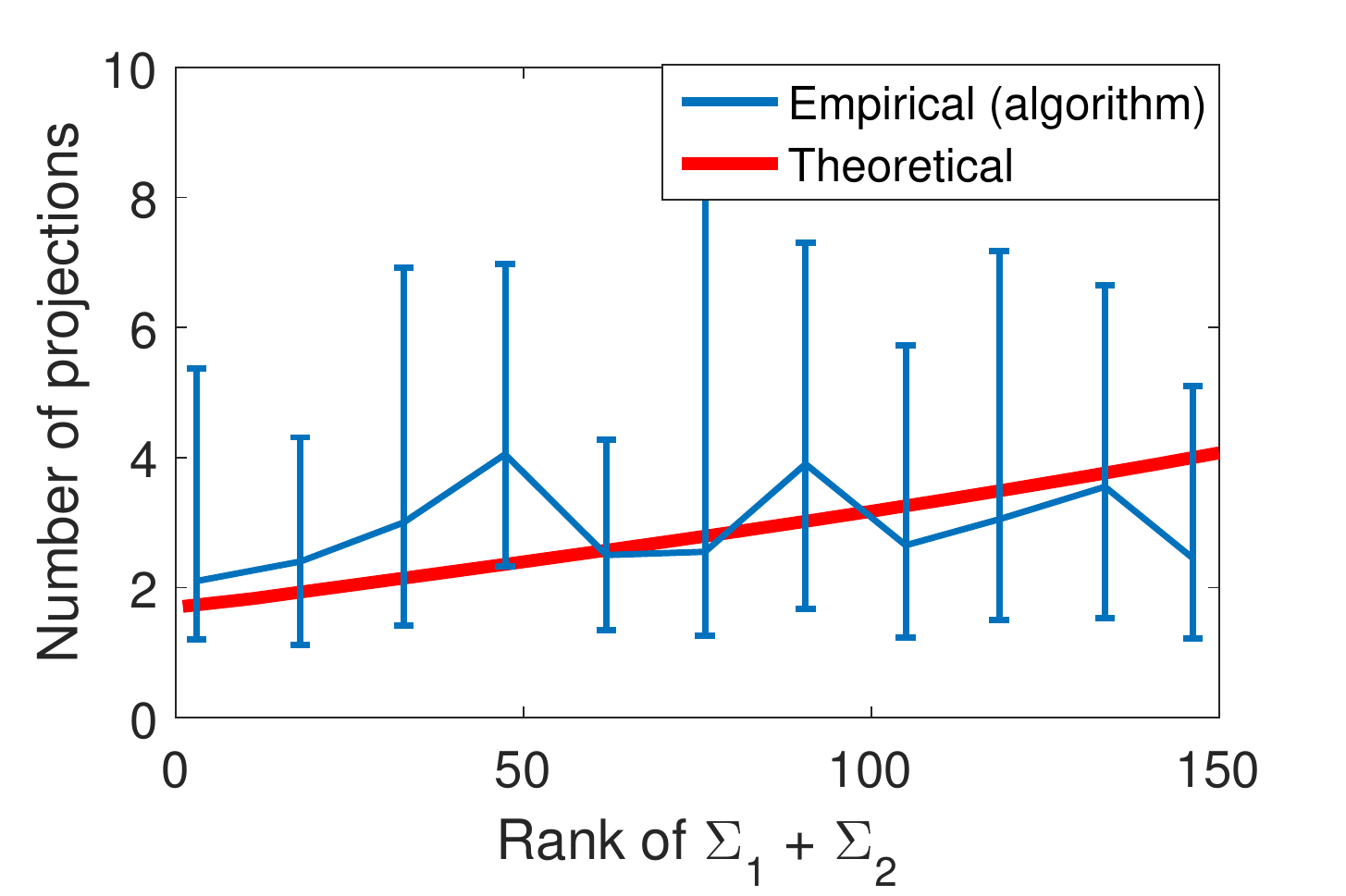,height=65mm,width=61mm}
\caption{Analysis of non-spherical mixtures with unequal covariances as a function of the rank of $(\Sigma_1+\Sigma_2)$. The data contains realization of mixtures of 10K points in $\mathds{R}^{1000}$, for $c=0.5$. \textbf{Left}: Algorithm's error vs. rank. \textbf{Right}: number of projections vs. rank compared with the theoretical bound provided by Theorem \ref{thm:2-r}.}
 \label{fig:nonSpher}
 \end{center}
 \end{figure}
%---------------------%----------------------------%-------------------------------%-------------------------------%
\section{Proofs}
\label{sec:proofs}
\subsection{Proof of Theorem \ref{thm:c-sep-vs-error}}
\begin{proof}
Since $\e(\alpha,\mv_1,\mv_2,\Sigma_1,\Sigma_2)$ corresponds to the error probability of an optimal classifier, the classification error of any (sub-optimal) classifier serves as an upper bound on $\e(\alpha,\mv_1,\mv_2,\Sigma_1,\Sigma_2)$. In particular, consider the hyperplane classifier that is orthogonal to $\mv_1-\mv_2$ and passes through ${\mv_1+\mv_2\over 2}.$ Characterizing the error probability of this specific classifier will yield our desired result. This hyperplane, because it is orthogonal to $\mv_1-\mv_2$  can be described as $(\mv_1-\mv_2)^T\xv=\beta$. Since it passes through $(\mv_1+\mv_2)/2$, we have
\[
\beta=(\mv_1-\mv_2)^T{\mv_1+\mv_2\over 2},
\]
or
\[
\beta={\|\mv_1\|^2-\|\mv_2\|^2\over 2}.
\]
This hyperplane labels every point $\xv\in\mathds{R}^p$ as follows:
\begin{enumerate}
\item If
\[
(\mv_1-\mv_2)^T\xv\geq {\|\mv_1\|^2-\|\mv_2\|^2\over 2},
\]
label $\xv$ as $1$.
\item If
\[
(\mv_1-\mv_2)^T\xv< {\|\mv_1\|^2-\|\mv_2\|^2\over 2},
\]
 label $\xv$ as $2$.
\end{enumerate}

 Let $\e_{h}(\alpha,\mv_1,\mv_2,\Sigma_1,\Sigma_2)$ denote the classification error of the described classifier. Clearly,
 \begin{align}
 \e(\alpha,\mv_1,\mv_2,\Sigma_1,\Sigma_2)\leq \e_h(\alpha,\mv_1,\mv_2,\Sigma_1,\Sigma_2).\label{eq:error-h-vs-opt}
 \end{align}
 To upper bound $\e_h(\alpha,\mv_1,\mv_2,\Sigma_1,\Sigma_2)$, note that under $(\mv_1-\mv_2)^T\Xv-\beta$, the first component of the original mixture of Gaussians in $\mathds{R}^p$ is  mapped into a Gaussian distribution in $\mathds{R}$ with mean
 \[
 (\mv_1-\mv_2)^T\mv_1-{\|\mv_1\|^2-\|\mv_2\|^2\over 2}={\|\mv_1-\mv_2\|^2\over 2},
 \]
 and variance
 \[
(\mv_1-\mv_2)^T\Sigma_1(\mv_1-\mv_2).
 \]
 Similarly, the second Gaussian is mapped to
 \[
 \Nc\Big(-{\|\mv_1-\mv_2\|^2\over 2},(\mv_1-\mv_2)^T\Sigma_2(\mv_1-\mv_2)\Big).
 \]
Clearly the weights  $(\alpha,1-\alpha)$ of the two Gaussian are preserved under this mapping.
Therefore, in summary, the classification error $\e_h(\alpha,\mv_1,\mv_2,\Sigma_1,\Sigma_2)$ can be written as
\begin{align}
\alpha Q({\|\mv_1-\mv_2\|^2\over 2\sqrt{(\mv_1-\mv_2)^T\Sigma_1(\mv_1-\mv_2)}} )+(1-\alpha)  Q({\|\mv_1-\mv_2\|^2\over 2\sqrt{(\mv_1-\mv_2)^T\Sigma_2(\mv_1-\mv_2)}} ).\label{eq:error-Q}
\end{align}
Note that, for $i=1,2$,
\[
(\mv_1-\mv_2)^T\Sigma_i(\mv_1-\mv_2)\leq \lambda_{\max}(\Sigma_i)\|\mv_1-\mv_2\|^2,
\]
or
\begin{align}
{\|\mv_1-\mv_2\|^2\over 2\sqrt{(\mv_1-\mv_2)^T\Sigma_i(\mv_1-\mv_2)}} \geq {\|\mv_1-\mv_2\|\over 2\sqrt{\lambda_{\max}(\Sigma_i)}}\geq  {\|\mv_1-\mv_2\|\over 2\sqrt{\lambda_{\max}(\Sigma_1)}+2\sqrt{\lambda_{\max}(\Sigma_2)}}.\label{eq:distance-in-1D}
\end{align}
On the other hand, since by assumption the two components  were $c$-separable in $\mathds{R}^p$,
\begin{align}
{\|\mv_1-\mv_2\|\over \sqrt{\lambda_{\max}(\Sigma_1)}+\sqrt{\lambda_{\max}(\Sigma_2)}}\geq c\sqrt{p}.\label{eq:connection-c-sep}
\end{align}
Therefore, since $Q(x)$ is a decreasing function of $x$,  combining \eqref{eq:error-Q}, \eqref{eq:distance-in-1D} and \eqref{eq:connection-c-sep} shows that
\[
\e_h(\alpha,\mv_1,\mv_2,\Sigma_1,\Sigma_2)\leq \alpha Q\left({c \over 2}\sqrt{p}\right)+(1-\alpha)Q\left({c \over 2}\sqrt{p}\right)= Q\left({c\over 2}\sqrt{p}\right).
\]
Finally, combining this upper bound with \eqref{eq:error-h-vs-opt} yields the desired result.

\end{proof}

%------------%------------%------------%------------%------------%------------
\subsection{Proof of Theorem \ref{thm:1}}
\begin{proof}
 Since the unitary vector $\Av/\|\Av\|$ is uniformly distributed over the unit sphere, we have
   \begin{align}\label{eq:1st-ineq-thm1}
    \P\Big(\Big|\langle{\mv_1-\mv_2\over \|\mv_1-\mv_2\|},{\Av\over \|\Av\|} \rangle\Big|> {\gamma(\sigma_1+\sigma_2)\over \|\mv_1-\mv_2\|}\Big)
&=  \P\Big(\Big|\langle (1,0,\ldots,0)^T,{\Av\over \|\Av\|} \rangle\Big|> {\gamma(\sigma_1+\sigma_2)\over \|\mv_1-\mv_2\|}\Big)\nonumber\\
&=  \P\Big({|A_1|\over \|\Av\|} > {\gamma(\sigma_1+\sigma_2)\over \|\mv_1-\mv_2\|}\Big).
 \end{align}
Therefore,  we are interested in deriving a lower bound on
 \begin{align}
 \P\Big({|A_1|\over \|\Av\|} > {\gamma(\sigma_1+\sigma_2)\over \|\mv_1-\mv_2\|}\Big).\label{eq:prob-simplified}
\end{align}
Note that, due to symmetry,  we have
\begin{align}
\E[{A_1^2\over \|\Av\|^2}]=\E[{A_2^2\over \|\Av\|^2}]=\ldots=\E[{A_p^2\over \|\Av\|^2}],\label{eq:exp-A1-An}
\end{align}
Moreover,
\begin{align}
\sum_{i=1}^p\E[{A_i^2\over \|\Av\|^2}]=\E\Big[{\sum_{i=1}^p  A_i^2\over \|\Av\|^2}\Big]=1.\label{eq:exp-sum-A1-An}
\end{align}
Therefore, combining \eqref{eq:exp-A1-An} and \eqref{eq:exp-sum-A1-An}, we have
\[
\E[{A_1^2\over \|\Av\|^2}]={1\over p}.
\]
So, intuitively, this  suggests that, if
\[
{\gamma(\sigma_1+\sigma_2)\over \|\mv_1-\mv_2\|}
\]
is much smaller that ${1\over p}$, the probability mentioned in \eqref{eq:prob-simplified}, which as we discussed earlier is an indicator of  separability, is not small.

On the other hand, replacing  $ {\gamma^2(\sigma_1+\sigma_2)^2 / \|\mv_1-\mv_2\|^2}$ by ${\alpha / n}$  in \eqref{eq:prob-simplified}, we have
\begin{align}
 \P\Big({A_1^2\over \|\Av\|^2} > {\gamma^2(\sigma_1+\sigma_2)^2\over \|\mv_1-\mv_2\|^2}\Big)&= \P\Big({A_1^2} >{\alpha \over p} \sum_{i=1}^pA_i^2\Big)\nonumber\\
 &= \P\Big((1-{\alpha\over p}){A_1^2} >{\alpha \over p} \sum_{i=2}^pA_i^2\Big).
\end{align}
But since $A_1,\ldots,A_p\stackrel{\rm i.i.d.}{\sim}\Nc(0,1)$, $\sum_{i=2}^p A_i^2$ has a chi-square distribution of order $n$. Then, for any  $\tau>0$, by Lemma 2 in \cite{JalaliM:11},
\begin{align}
\P\Big({1\over p-1}\sum_{i=2}^p A_i^2  \geq  1+\tau \Big)&\leq  {\rm e}^{-{p-1\over 2}(\tau-\log(1+\tau))}.
\end{align}
%and
%\begin{align}
%\P\Big({1\over p-1}\sum_{i=2}^p A_i^2 \leq  1-\tau \Big)&\leq  {\rm e}^{{p-1\over 2}(\tau+\log(1-\tau)}.
%\end{align}
Given $\tau>0$, define event $\Ec$ as
\[
\Ec\triangleq \{ {1\over p-1}\sum_{i=2}^p A_i^2 \leq  1+\tau \}.
\]
By the law of total probability,
\begin{align*}
\P\Big((1-{\alpha\over p}){A_1^2} >{\alpha \over p} \sum_{i=2}^pA_i^2\Big)&=\P\Big((1-{\alpha\over p}){A_1^2} >{\alpha \over p} \sum_{i=2}^pA_i^2,\Ec\Big)+\P\Big((1-{\alpha\over p}){A_1^2} >{\alpha \over p} \sum_{i=2}^pA_i^2,\Ec^c \Big)\nonumber\\
&\geq \P\Big((1-{\alpha\over p}){A_1^2} >{\alpha \over p} \sum_{i=2}^pA_i^2,\Ec\Big)\nonumber\\
&\geq  \P\Big((1-{\alpha\over p}){A_1^2} >{\alpha(n-1) \over p}(1+\tau) >{\alpha \over p} \sum_{i=2}^pA_i^2\Big)\nonumber\\
&\geq  \P\Big((1-{\alpha\over p}){A_1^2} >{\alpha(n-1) \over p}(1+\tau), 1+\tau >{1 \over p-1} \sum_{i=2}^pA_i^2\Big)\nonumber\\
&=  \P\Big((1-{\alpha\over p}){A_1^2} >\alpha(1-{1 \over p})(1+\tau)\Big)\P\Big({1 \over p-1} \sum_{i=2}^pA_i^2<1+\tau \Big),
\end{align*}
where the last line follows from the independence of $A_1$ and $(A_2,\ldots,A_p)$.
\end{proof}

%------------%------------%------------%------------%------------%------------
\subsection{Proof of Lemma \ref{lemma:5}}

\begin{proof}
Let $\Av=(A_1,\ldots,A_p)$ be generated $i.i.d.$~according to $\Nc(0,1)$. Then the separability of the two projected Gaussians under $\Av$ is equal to
\[
\gamma={|\langle \Av,\mv_1\rangle-\langle \Av,\mv_2\rangle|\over (\sigma_1+\sigma_2)\|\Av\| }.
\]
Therefore,
\begin{align*}
\gamma^2&={|\langle \Av,\mv_1-\mv_2\rangle|^2\over (\sigma_1+\sigma_2)^2 \|\Av\|^2 }\\
&={\|\mv_1-\mv_2\|^2\over (\sigma_1+\sigma_2)^2  }\left|\left\langle {\Av\over \|\Av\|},{\mv_1-\mv_2\over \|\mv_1-\mv_2\| }\right\rangle\right|^2.
\end{align*}
Since ${\Av\over \|\Av\|}$ is uniformly distributed under the unit sphere in $\mathds{R}^p$, in evaluating $\E{\gamma^2}$, without loss of generality we can assume that ${\mv_1-\mv_2\over \|\mv_1-\mv_2\| }=(1,0,\ldots,0)^T$. Therefore,
\begin{align*}
\E[\gamma^2]&={\|\mv_1-\mv_2\|^2\over (\sigma_1+\sigma_2)^2  }\E\left[\left|\left\langle {\Av\over \|\Av\|},{\mv_1-\mv_2\over \|\mv_1-\mv_2\| }\right\rangle\right|^2\right]\\
&={\|\mv_1-\mv_2\|^2\over (\sigma_1+\sigma_2)^2  }\E\left[{A_1^2\over  \|\Av\|^2}\right].
\end{align*}
But, as we showed in the proof of Theorem \ref{thm:1},
\[
\E\left[{A_1^2\over  \|\Av\|^2}\right]={1\over p}.
\]
Therefore, in summary,
\begin{align*}
\E[\gamma^2]&={\|\mv_1-\mv_2\|^2\over (\sigma_1+\sigma_2)^2 p  }=c^2.
\end{align*}
\end{proof}

%------------%------------%------------%------------%------------%------------
\subsection{Proof of Theorem \ref{thm:2}}
\begin{proof}
 Note that since $\Av^T\Sigma_1\Av\geq 0$ and $\Av^T\Sigma_2\Av\geq  0$, we always have
\[
\sqrt{\Av^T\Sigma_1\Av}+ \sqrt{\Av^T\Sigma_2\Av}\leq \sqrt{2\Av^T(\Sigma_1+\Sigma_2)\Av} .
\]
Therefore,
\begin{align}
&\P\Big(|\langle\mv_1-\mv_2,\Av\rangle|>\gamma (\sqrt{\Av^T\Sigma_1\Av}+ \sqrt{\Av^T\Sigma_2\Av}\;\;)\Big)\nonumber\\
&\geq \P\Big(|\langle\mv_1-\mv_2,\Av\rangle|>\gamma \sqrt{2\Av^T(\Sigma_1+\Sigma_2)\Av}\Big)\nonumber\\
&\geq \P\Big(|\langle\mv_1-\mv_2,\Av\rangle|>\gamma \sqrt{2\lambda_{\max}}\|\Av\|\Big),\label{eq:main-thm2-sum-sqrt}
\end{align}
where the last line follows because, for every $\Av$, $\Av^T(\Sigma_1+\Sigma_2)\Av\leq \lambda_{\max}\|\Av\|^2$.  Therefore, comparing \eqref{eq:main-thm2-sum-sqrt} with \eqref{eq:1st-ineq-thm1} reveals that the desired result follows similar to Theorem \ref{thm:1}, by replacing $\sigma_1+\sigma_2$ with $\sqrt{2\lambda_{\max}}$ .
\end{proof}

%------------%------------%------------%------------%------------%------------
\subsection{Proof of Lemma \ref{lemma:k-GMM}}
\begin{proof}
Define
\[
\phi_{(i,j)} \triangleq 2Q \left({\gamma_{\min}\over c_{(i,j)}}\sqrt{{1-{1 \over p}\over 1-{\gamma_{\min}^2\over  c^2_{(i,j)}p}}(1+\tau)}\;\right)(1- {\rm e}^{-{p-1\over 2}(\tau-\log(1+\tau))}).
\]
By the union bound,
  \begin{align}\label{union}
\P(\Bc^c) &\leq \sum_{i=1}^k\sum_{j=i+1}^k\P\left(\left\{\left|\left\langle\mv_i-\mv_j,{\Av\over \|\Av\|}\right\rangle\right|\leq  \gamma_{\min} (\sigma_i+\sigma_j)\right\} \right)\nonumber\\
&\stackrel{(a)}{\leq} \sum_{i=1}^k\sum_{j=i+1}^k (1-\phi_{(i,j)})\nonumber\\
&\leq {k^2\over 2}\left(1-\max_{(i,j)}\{\phi_{(i,j)}\}\right),
\end{align}
where $(a)$ follows from   Theorem \ref{thm:1} and the fact that for $i,j\in\{1,\ldots,k\}$
\[
{\gamma_{\min}^2(\sigma_i+\sigma_j)^2 p\over \|\mv_i-\mv_j\|^2}= {\gamma_{\min}^2\over c_{(i,j)}^2}.
\]
For $\tau=0.1$, $(\tau-\log(1+\tau))/2\geq 0.002$. Therefore, setting $\tau=0.1$ in \eqref{lemma:k-GMM} and noting that $Q$ function is a monotonically decreasing function of its argument, it follows that
\begin{align}
\phi_{(i,j)} &\geq  2Q \left({\gamma_{\min}\over c_{(i,j)}}\sqrt{{1.1(1-{1 \over p})\over (1-{\gamma^2_{\min}\over c^2_{(i,j)} p})}}\;\right)(1- {\rm e}^{-0.002 p})\nonumber\\
 &\geq 2Q \left({\gamma_{\min}\over c_{\min}}\sqrt{{1.1\over 1-{\gamma^2_{\min}\over c^2_{\min} p}}}\;\right)(1- {\rm e}^{-0.002 p}),\label{eq:bound-phi-ij}
 \end{align}
 where the last inequality holds  because
 \[
 {{(1-{1 \over p})\over (1-{\gamma^2_{\min}\over c^2_{(i,j)} p})}}\leq {1\over  1-{\gamma^2_{\min}\over c^2_{\min} p}}.
 \]
 Therefore, taking the maximum of the both sides of \eqref{eq:bound-phi-ij}, it follows that
 \[
 \max_{i,j}\phi_{(i,j)}\geq 2Q \left({\gamma_{\min}\over c_{\min}}\sqrt{{1.1\over 1-{\gamma^2_{\min}\over c^2_{\min} p}}}\;\right)(1- {\rm e}^{-0.002 p}).
 \]
\end{proof}
%------------%------------%------------%------------%------------%------------
\subsection{Proof of Theorem \ref{thm:2-r}}

\begin{proof}
 Note that since $\Av^T\Sigma_1\Av\geq 0$ and $\Av^T\Sigma_2\Av\geq  0$, we always have
\[
\sqrt{\Av^T\Sigma_1\Av}+ \sqrt{\Av^T\Sigma_2\Av}\leq \sqrt{2\Av^T(\Sigma_1+\Sigma_2)\Av} .
\]
Therefore,
\begin{align}
&\P\Big(|\langle\mv_1-\mv_2,\Av\rangle|>\gamma (\sqrt{\Av^T\Sigma_1\Av}+ \sqrt{\Av^T\Sigma_2\Av}\;\;)\Big)\nonumber\\
&\geq \P\Big(|\langle\mv_1-\mv_2,\Av\rangle|>\gamma \sqrt{2\Av^T(\Sigma_1+\Sigma_2)\Av}\Big)
\end{align}
Since $\Sigma_1+\Sigma_2$ is  always a semi-positive definite  matrix, it can be decomposed as
\[
\Sigma_1+\Sigma_2=P^TDP,
\]
where $P\in\mathds{R}^{p\times p}$ is an orthogonal matrix ($P^TP=I_p$), and $D\in\mathds{R}^{p\times p}$ is a diagonal matrix whose diagonal entries are non-negative. Let
\[
D=\diag(\l_1,\ldots, \l_p),
\]
where $\l_i\geq 0$, for all $i$. Using this decomposition, $\Av^T(\Sigma_1+\Sigma_2)\Av$ can be written as
\[
\Av^T(\Sigma_1+\Sigma_2)\Av=(P\Av)^TDP\Av.
\]
Let $\Bv\triangleq P\Av$. Since $P$ is an orthogonal matrix, $\Bv$ is still  distributed as $\Av$, \ie $B_1,\ldots,B_p$ are i.i.d.~$\Nc(0,1)$. By this change of variable, the probability mentioned in  \eqref{eq:prob-sep-condition-memory} can be written as
\begin{align}
\P&\Big(|\langle\mv_1-\mv_2,{\Av}\rangle|\geq  \gamma (\sqrt{\Av^T\Sigma_1\Av}+ \sqrt{\Av^T\Sigma_2\Av}\;)\Big)\nonumber\\
&\geq
\P\Big(|\langle\mv_1-\mv_2,P^{-1}\Bv\rangle|>\gamma \sqrt{2\Bv^TD\Bv}\Big)\nonumber\\
&=\P\Big(|\langle P(\mv_1-\mv_2), \Bv\rangle|>\sqrt{2\gamma^2} \| D^{1\over 2}\Bv\|\Big).\label{eq:m1-m2-A-to-B}
\end{align}
Note that sine by assumption $\rank(\Sigma_1+\Sigma_2)=r$,  $\Sigma_1+\Sigma_2$ has only $r$ non-zero eigenvalues. Define $\Cv\in\mathds{R}^p$ such that, for $i=1,\ldots,p$,
\[
C_i=B_i\ind_{\lambda_i\neq 0}.
\]
That is, for every $\lambda_i\neq 0$, $C_i$ is equal to $B_i$. For every $\lambda_i=0$, $C_i=0$.
Using this definition, $D^{1\over 2}\Bv=D^{1\over 2}\Cv$. Note that
\begin{align}
\|D^{1\over 2}\Cv\|\leq \sqrt{\lambda_{\max}}\|\Cv\|.\label{eq:B-to-C-norm}
\end{align}
Combining \eqref{eq:m1-m2-A-to-B} and \eqref{eq:B-to-C-norm}, it follows that
\begin{align}
\P\Big(|\langle\mv_1-\mv_2,{\Av}\rangle|\geq  \gamma (\sqrt{\Av^T\Sigma_1\Av}+ \sqrt{\Av^T\Sigma_2\Av}\;)\Big)&\geq \P\Big(|\langle P(\mv_1-\mv_2), \Bv\rangle|>\sqrt{2\gamma^2\lambda_{\max}} \|\Cv\|\Big)\nonumber\\
&= \P\Big(|\langle P(\mv_1-\mv_2), {\Bv\over \|\Bv\|}\rangle|>\sqrt{2\gamma^2\lambda_{\max}} {\|\Cv\|\over \|\Bv\|}\Big).
\end{align}
Given $\tau_1>0$ and $\tau_2>0$, define events $\Ec_1$ and $\Ec_2$ as
\[
\Ec_1\triangleq \{\|\Bv\|^2\geq p(1-\tau_1)\},
\]
and
\[
\Ec_2\triangleq \{\|\Cv\|^2\leq r(1+\tau_2)\},
\]
respectively. Note that, conditioned on $\Ec_1\cap\Ec_2$,
\[
{\|\Cv\|\over \|\Bv\|}\leq \sqrt{r(1+\tau_2)\over p(1-\tau_1)}.
\] Therefore,
\begin{align}
\P\Big(|\langle P(\mv_1-\mv_2), {\Bv\over \|\Bv\|}\rangle|\leq \sqrt{2\gamma^2\lambda_{\max}} {\|\Cv\|\over \|\Bv\|}\Big) =& \P\Big(|\langle P(\mv_1-\mv_2), {\Bv\over \|\Bv\|}\rangle|\leq \sqrt{2\gamma^2\lambda_{\max}} {\|\Cv\|\over \|\Bv\|},\Ec_1\cap\Ec_2\Big)\nonumber\\
&+\P\Big(|\langle P(\mv_1-\mv_2), {\Bv\over \|\Bv\|}\rangle|\leq \sqrt{2\gamma^2\lambda_{\max}} {\|\Cv\|\over \|\Bv\|},(\Ec_1\cap\Ec_2)^c\Big)\nonumber\\
\leq & \P\left(|\langle P(\mv_1-\mv_2), {\Bv\over \|\Bv\|}\rangle|\leq \sqrt{2\gamma^2\lambda_{\max}(1+\tau_2)r\over (1-\tau_1)p} \;\right)\nonumber\\
&+\P(\Ec_1^c)+\P(\Ec_2^c).
\end{align}
But, from Lemma 2 in \cite{JalaliM:11},
\begin{align}
\P(\Ec_1^c)\leq   {\rm e}^{{p\over 2}(\tau_1+\log(1-\tau_1))},
\end{align}
and
\[
\P(\Ec_2^c)\leq {\rm e}^{-{r\over 2}(\tau_2-\log(1+\tau_2))}.
\]
Also, note that since $P$ is an orthogonal matrix, $\|P(\mv_1-\mv_2)\|=\|\mv_1-\mv_2\|$. Therefore, the desired result follows by comparing $\P(|\langle P(\mv_1-\mv_2), {\Bv\over \|\Bv\|}\rangle|\leq \sqrt{2\gamma^2\lambda_{\max}(1+\tau_2)r\over (1-\tau_1)p} )$ with \eqref{eq:1st-ineq-thm1} and using the result of Theorem \ref{thm:1}.
 \end{proof}

\subsection{Proof of Corollary \ref{cor:num_proj4beta}}
\begin{proof}
We note that in (\ref{eq:prob_nonsphr_frnk}) for any $\tau>0$,  $\lim_{p \rightarrow \infty} (1- {\rm e}^{-{p-1\over 2}(\tau-\log(1+\tau))})= 1$. Since
\[c \triangleq{\|\mv_1-\mv_2\|\over  \sqrt{p} \left(\sqrt{\lambda_{\max}(\Sigma_1)}+\sqrt{\lambda_{\max}(\Sigma_2)}\right)},
\]
$\beta$ can be expressed as
\[\beta=\frac{2\gamma^2\lambda_{max}}{c^2(\sqrt{\lambda_{\max}(\Sigma_1)}+\sqrt{\lambda_{\max}(\Sigma_2)})}.
 \]
 Next, since $\lambda_{\max}(\Sigma_1+\Sigma_2)\leq \lambda_{\max}(\Sigma_1)+\lambda_{\max}(\Sigma_2)$ we observe that $\beta \leq \frac{2\gamma^2}{c^2}$,
for all p. Therefore
\[
\lim_{p \rightarrow \infty}\frac{1-\frac{1}{p}}{1-\frac{\beta}{p}}=1
\]
We obtain that for any $\tau>0$
\begin{align}
\lim_{p \rightarrow \infty}\P\Big({A_1^2} >\beta {(1-{1 \over p})\over (1-{\beta\over p})}(1+\tau)\Big)(1- {\rm e}^{-{p-1\over 2}(\tau-\log(1+\tau))}) =  \P\Big({A_1^2} >\beta (1+\tau)\Big).
\end{align}
Since $\tau$ is a free parameter, we have
\begin{align}
\lim_{p \rightarrow \infty}d(\gamma) \leq \frac{1}{\P\Big({A_1^2} >\beta \Big)} = \frac{1}{2Q(\sqrt{\beta})}
\end{align}
\end{proof}

\subsection{Proof of Lemma \ref{lemma:est-error-to-class-error}}
\begin{proof}
As shown in the proof of Lemma \ref{lemma:lb-on-e-opt}, the optimal Bayesian classifier breaks the real line at $t_{\rm opt}={\mu_1+\mu_2\over 2}-{\sigma^2\over (\mu_1-\mu_2)}\ln {w\over 1-w},$ and achieves a classification error equal to
\begin{align*}
e_{\rm opt}&=wQ\left({t_{\rm opt}-\mu_1\over \sigma_1}\right)+(1-w)Q\left({\mu_2-t_{\rm opt}\over \sigma_2}\right)\nonumber\\
&=wQ\left(\gamma+{1\over 2\gamma}\ln {w\over 1-w}\right)+(1-w)Q\left(\gamma-{1\over 2\gamma}\ln {w\over 1-w}\right).
\end{align*}
On the other hand, without having access to the exact parameters, a clustering algorithm that operates based on   the estimated values $(\hat{\mu}_1,\hat{\mu}_2,\hat{\sigma}_1,\hat{\sigma}_1,\hat{w})$  finds $\hat{t}_1$ and $\hat{t}_2$, which are the solutions of  ${\hat{w}\over \sqrt{2\pi \hat{\sigma}_1^2}}{\rm e}^{-(t-\muh_1)^2 \over 2\hat{\sigma}_1^2}={1-\hat{w}\over \sqrt{2\pi \hat{\sigma}_2^2}}{\rm e}^{-(t-\muh_2)^2 \over 2\hat{\sigma}_2^2}$, and puts the decision boundary points at these two points.  For $i=1,2$, let
\[
\tilde{t}_i\triangleq \hat{t}_i-\muh_1.
\]
and
\[
(s_i,\sh_i)\triangleq \left({1\over \sigma_i^2},{1\over \hat{\sigma}_i^2}\right).
\]
Note that $\sigma_1=\sigma_2$ by assumptions. Therefore $s_1=s_2$. Let
\[
s\triangleq s_1=s_2,
\]
and
\[
\Dhmu\triangleq \muh_2-\muh_1.
\]
Using the mentioned change of variable, $(\tilde{t}_1,\tilde{t}_2)$ are the solutions of the following second order equation
\begin{align}
(\sh_1-\sh_2)x^2+2\Dhmu\sh_2x-\Dhmu^2\hat{s}_2+2\ln {\hat{s}_2\over \hat{s}_1}-2\ln{\hat{w}\over 1-\hat{w}}=0,\label{eq:2nd-order-s1-equal-s2}
\end{align}
Assume that $\tilde{t}_1$ denotes the point that approximates  $t_{\rm opt}-\mu_1$. A clustering algorithm that decides  based on these estimated boundary points estimates its achieved error as $\hat{e}_{\rm opt}$, where, if $\hat{\sigma}_1\leq \hat{\sigma}_2$,
\[
\hat{e}_{\rm opt}=\hat{w}\left(Q\left({\tilde{t}_1\over \hat{\sigma}_1}\right)+Q\left(-{\tilde{t}_2\over \hat{\sigma}_1}\right)\right)+(1-\hat{w})\left(Q\left({\Dhmu-\tilde{t}_1\over \hat{\sigma}_2}\right)-Q\left({\Dhmu-\tilde{t}_2\over \hat{\sigma}_2}\right)\right),
\]
and if $\hat{\sigma}_1> \hat{\sigma}_2$,
\[
\hat{e}_{\rm opt}=\hat{w}\left(Q\left({\tilde{t}_1\over \hat{\sigma}_1}\right)-Q\left({\tilde{t}_2\over \hat{\sigma}_1}\right)\right)+(1-\hat{w})\left(Q\left({\Dhmu-\tilde{t}_1\over \hat{\sigma}_2}\right)+Q\left({\tilde{t}_2-\Dhmu\over \hat{\sigma}_2}\right)\right).
\]
Since for all $x$ and $x'$, $|Q(x)-Q(x')|\leq |x-x'|$,  if  $\hat{\sigma}_1\leq \hat{\sigma}_2$,
\begin{align*}
|e_{\rm opt}-\hat{e}_{\rm opt}|&\leq  |w-\hat{w}|+\left|{\tilde{t}_1\over \hat{\sigma}_1}-{t_{\rm opt}-\mu_1\over \sigma_1}\right|+\left|{\Dhmu-\tilde{t}_1\over \hat{\sigma}_2}-{\mu_2-t_{\rm opt}\over \sigma_2}\right|+Q\left(-{\tilde{t}_2\over \hat{\sigma}_1}\right),
\end{align*}
and if  $\hat{\sigma}_1> \hat{\sigma}_2$,
\begin{align*}
|e_{\rm opt}-\hat{e}_{\rm opt}|&\leq  |w-\hat{w}|+\left|{\tilde{t}_1\over \hat{\sigma}_1}-{t_{\rm opt}-\mu_1\over \sigma_1}\right|+\left|{\Dhmu-\tilde{t}_1\over \hat{\sigma}_2}-{\mu_2-t_{\rm opt}\over \sigma_2}\right|+Q\left({\tilde{t}_2-\Dhmu\over \hat{\sigma}_2}\right).
\end{align*}
Note that, by the triangle inequality,
\begin{align}
\left|{\tilde{t}_1\over \hat{\sigma}_1}-{t_{\rm opt}-\mu_1\over \sigma_1}\right|&\leq {|\tilde{t}_1-{t_{\rm opt}+\mu_1}|\over \hat{\sigma}_1}+|t_{\rm opt}-\mu_1|\left|{1\over \hat{\sigma}_1}-{1\over \sigma_1}\right|.
\end{align}
Similarly,
\begin{align}
\left|{\Dhmu-\tilde{t}_1\over \hat{\sigma}_2}-{\mu_2-t_{\rm opt}\over \sigma_2}\right| \leq  {|\Dhmu-\tilde{t}_1-\mu_2+t_{\rm opt}|\over \hat{\sigma}_2}+|t_{\rm opt}-\mu_2|\left|{1\over \hat{\sigma}_2}-{1\over \sigma_2}\right|.
\end{align}
But, by assumption, $|\sigma_i^2-\hat{\sigma}_i^2|\leq \epsilon\dmu^2$. Therefore, $\hat{\sigma}_i\leq \sigma_i\sqrt{1+\epsilon\dmu^2/\sigma_i^2}=\sigma_i\sqrt{1+4\gamma^2\epsilon}\leq \sigma_i(1+2\gamma^2\epsilon)$. Similarly, $\hat{\sigma}_i\geq  \sigma_i\sqrt{1-4\gamma^2 \epsilon}\geq \sigma_i(1-4\gamma^2 \epsilon)$. Hence, $|\sigma_i-\hat{\sigma}_i|\leq 4c^2\epsilon$ and
\[
\left|{1\over \hat{\sigma}_i}-{1\over \sigma_i}\right|\leq {4\gamma^2\epsilon \over (1-4\gamma^2\epsilon) \sigma_i}.
\]
Also, note that since $t_{\rm opt}={\mu_1+\mu_2\over 2}-{\sigma_1^2\over (\mu_1-\mu_2)}\ln {w\over 1-w},$ for $i=1,2$,
\begin{align*}
{|t_{\rm opt}-\mu_i|\over \sigma_1}\leq \gamma+{1\over 2\gamma}\ln {1-w_{\min}\over w_{\min}}.
\end{align*}
In summary,  if  $\hat{\sigma}_1\leq \hat{\sigma}_2$,
\begin{align}
|e_{\rm opt}-\hat{e}_{\rm opt}|&\leq  \left( 1+ 4\gamma^3+2\gamma\ln {1-w_{\min}\over w_{\min} }\right)\epsilon+ {1\over {\sigma}_1}|\tilde{t}_1-{t_{\rm opt}+\mu_1}|+Q\left(-{\tilde{t}_2\over \hat{\sigma}_1}\right)+o(\epsilon),\label{eq:e-min-e-opt-1}
\end{align}
and if  $\hat{\sigma}_1> \hat{\sigma}_2$,
\begin{align}
|e_{\rm opt}-\hat{e}_{\rm opt}|&\leq  \left( 1+ 4\gamma^3+2\gamma\ln {1-w_{\min}\over w_{\min} }\right)\epsilon+ {1\over {\sigma}_1}|\tilde{t}_1-{t_{\rm opt}+\mu_1}|+ Q\left({\tilde{t}_2-\Dhmu\over \hat{\sigma}_2}\right)+o(\epsilon).\label{eq:e-min-e-opt-2}
\end{align}

In the rest of the proof, we mainly focus on bounding $|\tilde{t}_1-{t_{\rm opt}+\mu_1}|$. Since $\tilde{t}_1$ and $\tilde{t}_2$ are the solutions of \eqref{eq:2nd-order-s1-equal-s2}, they can be computed as
\[
\tilde{t}_1,\tilde{t}_2={-\Dhmu\sh_2 \pm\sqrt{\Delta}\over (\sh_1-\sh_2)},
\]
where
\begin{align*}
\Delta&=(\Dhmu\sh_2)^2-(\sh_1-\sh_2)\left(-\Dhmu^2\hat{s}_2+2\ln {\hat{s}_2\over \hat{s}_1}-2\ln{\hat{w}\over 1-\hat{w}}\right).
\end{align*}
Define $\upsilon$  as
%\begin{align}
%\rho_i\triangleq \hat{\sigma}_i-\sigma_i,\label{eq:def-sigma-hat}
%\end{align}
%and
\begin{align}
\upsilon\triangleq \muh_2-\muh_1-(\mu_2-\mu_1).\label{eq:def-mu-hat}
\end{align}
Note that since by assumption $|\muh_i-\mu_i|\leq \epsilon \dmu$, where
\[
\dmu \triangleq |\mu_2-\mu_1|,
\]
we have
\begin{align*}
|\upsilon|\leq 2\dmu\epsilon,
\end{align*}
Define $\tau_1$ and $\tau_2$ as
\[
\tau_i\triangleq \sh_i-s.
\]
 Note that
\begin{align}
|\tau_i|\leq {\dmu^2\epsilon\over \hat{\sigma}_i^2\sigma_i^2}\leq {\dmu^2\epsilon\over \sigma_i^2(\sigma_i^2-\epsilon\dmu^2)}=  {4 \gamma^2 s \epsilon\over 1-4 \epsilon \gamma^2}\leq  {4 \gamma^2 s\epsilon \over 1-4 \epsilon \gamma_{\max}^2}\leq 8 \gamma^2 s\epsilon,\label{eq:bd-tau-i}
\end{align}
where the last inequality holds as long as $4  \gamma_{\max}^2\epsilon\leq {1\over 2}$.

Define $\varepsilon$ as
\begin{align}
\varepsilon\triangleq {\sh_2-\sh_1\over (\Dhmu \sh_2)^2}\left(-\Dhmu^2\hat{s}_2+2\ln {\hat{s}_2\over \hat{s}_1}-2\ln{\hat{w}\over 1-\hat{w}}\right).
\label{eq:def-varepsilon}
\end{align}
%Further, assume that the error in estimating the parameters is small enough, such that
%\begin{align}
%\muh_1\sh_1<\muh_2\sh_2.\label{eq:cond-error-mu-times-s}
%\end{align}
Then, using this definition,  it follows from \eqref{eq:2nd-order-s1-equal-s2} that
\begin{align}
\tilde{t}_1={\Dhmu\sh_2 \over \sh_2-\sh_1}\left(1
-\sqrt{1+\varepsilon}\;\right),\label{eq:that-1}
\end{align}
and
\begin{align}
\tilde{t}_2={\Dhmu\sh_2 \over \sh_2-\sh_1}\left(1
+\sqrt{1+\varepsilon}\;\right).\label{eq:that-2}
\end{align}
Define function  $f$ as $f(x)=\sqrt{1+x}$.
Then, using the Taylor expansion of function $f$ around zero,
\begin{align}
f(\varepsilon)=1+{1\over 2}\varepsilon+{f''(r)\over 2}\varepsilon^2,\label{eq:taylor-expansion}
\end{align}
where $|r|\leq |\varepsilon|$. Note that
%Combining \eqref{eq:taylor-expansion} with \eqref{eq:that-1} and \eqref{eq:that-2}, and noting that
\begin{align}
{\Dhmu\sh_2 \over \sh_2-\sh_1} \varepsilon&={1\over \Dhmu \sh_2}\left(-\Dhmu^2\hat{s}_2+2\ln {\hat{s}_2\over \hat{s}_1}-2\ln{\hat{w}\over 1-\hat{w}}\right).\nonumber\\
&=-\Dhmu +{2 \over \Dhmu \sh_2}\left(\ln {\hat{s}_2\over \hat{s}_1}-\ln{\hat{w}\over 1-\hat{w}}\right).
%{s(\hat{\mu}_1^2-\hat{\mu}_2^2)+\eta-2\ln{\hat{w}\over 1-\hat{w}} \over \muh_2\sh_2-\muh_1\sh_1},
\end{align}
%and
%\begin{align*}
%{(\muh_1\sh_1-\muh_2\sh_2) \over \tau_1-\tau_2}\varepsilon^2&={(s(\hat{\mu}_1^2-\hat{\mu}_2^2)+\eta-2\ln{\hat{w}\over 1-\hat{w}})^2 \over (\muh_1\sh_1-\muh_2\sh_2)^3}(\tau_1-\tau_2),
%\end{align*}
Therefore, we have
\begin{align}
\tilde{t}_1&={\Dhmu\over 2} +{1 \over \Dhmu \sh_2}\left(\ln{\hat{w}\over 1-\hat{w}}-\ln {\hat{s}_2\over \hat{s}_1}\right)-{f''(r)\over 2}\left({\Dhmu\sh_2 \over \sh_2-\sh_1}\right)\varepsilon^2.\label{eq:derived-t-tilde-1}
%\\&{s(\hat{\mu}_1^2-\hat{\mu}_2^2)+\eta-2\ln{\hat{w}\over 1-\hat{w}} \over 2(\muh_1\sh_1-\muh_2\sh_2)}-f''(r){(s(\hat{\mu}_1^2-\hat{\mu}_2^2)+\eta-2\ln{\hat{w}\over 1-\hat{w}})^2 \over 2(\muh_1\sh_1-\muh_2\sh_2)^3}(\tau_1-\tau_2),\label{eq:that-1-simplified}
\end{align}
and
\begin{align}
\tilde{t}_2={2\Dhmu\sh_2 \over \sh_2-\sh_1}-\tilde{t}_1.\label{eq:t-tilde2}
\end{align}
As a reminder  $t_{\rm opt}={\mu_1+\mu_2\over 2}-{1\over (\mu_1-\mu_2)s}\ln {w\over 1-w}$. Therefore, from \eqref{eq:derived-t-tilde-1}, we have
\begin{align}
\left|\tilde{t}_1-t_{\rm opt}+\mu_1\right| &= \left|{\Dhmu\over 2} +{1 \over \Dhmu \sh_2}\left(\ln{\hat{w}\over 1-\hat{w}}-\ln {\hat{s}_2\over \hat{s}_1}\right)-{f''(r)\over 2}\left({\Dhmu\sh_2 \over \sh_2-\sh_1}\right)\varepsilon^2-t_{\rm opt}+\mu_1\right|\nonumber\\
&\leq   \dmu \epsilon+ \left|{1 \over \Dhmu \sh_2}\ln{\hat{w}\over 1-\hat{w}}-{1\over s (\mu_2-\mu_1)}\ln {w\over 1-w}\right|+\left|{1 \over \Dhmu \sh_2}\ln {\hat{s}_2\over \hat{s}_1}\right|+\left|{f''(r)\over 2}\left({\Dhmu\sh_2 \over \sh_2-\sh_1}\right)\varepsilon^2\right|.\label{eq:diff-t-opt-t-tilde}
\end{align}
We next bound the error terms in \eqref{eq:diff-t-opt-t-tilde}. Note that, by the triangle inequality,
\begin{align}
\left|{1 \over \Dhmu \sh_2}\ln{\hat{w}\over 1-\hat{w}}-{1\over s (\mu_2-\mu_1)}\ln {w\over 1-w}\right|&=\left|{1 \over \Dhmu \sh_2}\left(\ln{\hat{w}\over 1-\hat{w}}-\ln{w\over 1-w}+\ln{w\over 1-w}\right)-{1\over s (\mu_2-\mu_1)}\ln {w\over 1-w}\right|\nonumber\\
&\leq {1 \over |\Dhmu| \sh_2}\left| \ln{\hat{w}\over 1-\hat{w}}-\ln{w\over 1-w} \right|+\left|\ln {w\over 1-w}\right|\left| {1 \over \Dhmu \sh_2}- {1\over s (\mu_2-\mu_1)}\right|\label{eq:term000}
\end{align}

Since $|\mu_i-\muh_i|\leq \epsilon\dmu$, $|\Dhmu|=|\muh_1-\muh_2|\geq \dmu(1-2\epsilon)$. Therefore, we have
\begin{align}
{1\over |\Dhmu|}&\leq {1\over \dmu(1-2\epsilon)}.\label{eq:term1-bd}
\end{align}
Let $g(w)=\ln {w\over 1-w}$. Then, $g'(w)={1\over w}+{1\over 1-w}$. Therefore, since by assumption, $|w-\hat{w}|\leq \epsilon$, we have
\begin{align}
\left|\ln{\hat{w}\over 1-\hat{w}}-\ln{{w}\over 1-{w}}\right|\leq \left({1\over w_{\min}}+{1\over 1-w_{\min}}\right)\epsilon\leq {2\epsilon\over w_{\min}}.\label{eq:term2-bd}
\end{align}
Note that since $w\in(w_{\min},0.5)$, and since $w\over 1-w$ is an increasing function of $w$ in this interval, we have
\begin{align}
|\ln{w\over 1-w}|\leq \ln {1-w_{\min}\over w_{\min}}.\label{eq:term3-bd}
\end{align}
Note that
\[
\ln{\sh_2\over \sh_1}=\ln{s+\tau_1\over s+\tau_2}=\ln{1+\tau_1/s\over 1+\tau_2/s}.
\]
Hence,
\begin{align}
\left|\ln{\sh_2\over \sh_1}\right|&\leq \ln{1+|{\tau_1\over s}|\over 1-|{\tau_2\over s}|}\leq\ln  {1+{4\epsilon \gamma^2  \over 1-4 \epsilon \gamma_{\max}^2}\over 1- {4\epsilon \gamma^2  \over 1-4 \epsilon \gamma_{\max}^2}}=\ln  {1\over   1-8 \epsilon \gamma_{\max}^2}\leq {8\gamma_{\max}^2\epsilon\over   1-8 \epsilon \gamma_{\max}^2}\leq 16\gamma_{\max}^2\epsilon, \label{eq:term6-bd}
\end{align}
where the last line holds if $8\gamma_{\max}^2\epsilon<{1\over 2}$.
Combining \eqref{eq:term000}, \eqref{eq:term1-bd}, \eqref{eq:term2-bd}, \eqref{eq:term3-bd} and \eqref{eq:term6-bd} with \eqref{eq:diff-t-opt-t-tilde}, it follows that
\begin{align}
\left|\tilde{t}_1-t_{\rm opt}+\mu_1\right| \leq&   \dmu \epsilon+ {2\epsilon  \over (1-2\epsilon) (1-8 \gamma^2 s\epsilon)w_{\min}\dmu s }+ {1\over \sh_2}\ln {1-w_{\min}\over w_{\min}}\left| {1 \over \Dhmu }- {1\over  (\mu_2-\mu_1)}\right|\nonumber\\
&+{1\over \dmu} \ln {1-w_{\min}\over w_{\min}}\left| {1 \over  \sh_2}- {1\over s}\right|+{16\gamma_{\max}^2\epsilon   \over (1-2\epsilon) (1-8 \gamma^2 s\epsilon)\dmu s } +\left|{f''(r)\over 2}\left({\Dhmu\sh_2 \over \sh_2-\sh_1}\right)\varepsilon^2\right|\nonumber\\
\leq&   \dmu \epsilon+ {2\epsilon  \over (1-2\epsilon) (1-8 \gamma^2 s\epsilon)w_{\min}\dmu s }+ \left({2\epsilon \over (1-2\epsilon)(1-8\gamma^2\epsilon)\dmu s}+\epsilon\dmu \right)\ln {1-w_{\min}\over w_{\min}}\nonumber\\
&+{16\gamma_{\max}^2\epsilon   \over (1-2\epsilon) (1-8 \gamma^2 s\epsilon)\dmu s } +\left|{f''(r)\over 2}\left({\Dhmu\sh_2 \over \sh_2-\sh_1}\right)\varepsilon^2\right|\nonumber\\
=&  \left( \dmu +{2\over w_{\min}\dmu s}+\left({2\over \dmu s}+\dmu\right) \ln {1-w_{\min}\over w_{\min}}+{16\gamma_{\max}^2   \over \dmu s }\right) \epsilon+{f''(r)\over 2}\left|{\Dhmu\sh_2 \over \sh_2-\sh_1}\right|\varepsilon^2+o(\epsilon).\label{eq:diff-t-opt-t-tilde-level2}
\end{align}

Finally, we need to bound $\varepsilon$, defined in \eqref{eq:def-varepsilon}. By the triangle inequality and \eqref{eq:term1-bd}, it follows that.
\begin{align}
|\varepsilon|\leq&{|\sh_2-\sh_1| \over  \dmu^2(1-2\epsilon)^2 (\sh_2)^2}\left( \dmu^2(1+2\epsilon)^2\sh_2 +2\left|\ln {\hat{s}_2\over \hat{s}_1}\right|+2\left|\ln{\hat{w}\over 1-\hat{w}}\right|\right)\nonumber\\
=&{|\sh_2-\sh_1| \over  \dmu^2(1-2\epsilon)^2 (\sh_2)^2}\left( \dmu^2(1+2\epsilon)^2\sh_2 +2\left|\ln {\hat{s}_2\over \hat{s}_1}\right|+2\left|\ln{\hat{w}\over 1-\hat{w}}-\ln{w\over 1-w}+\ln{w\over 1-w}\right|\right)\nonumber\\
\stackrel{(a)}{\leq} &{|\sh_2-\sh_1| \over  \dmu^2(1-2\epsilon)^2 (1-8\gamma^2 \epsilon)^2s^2}\left( \dmu^2(1+2\epsilon)^2(1+8\gamma^2 \epsilon)s +32\gamma_{\max}^2\epsilon+{4\epsilon \over w_{\min}}+2\ln {1-w_{\min}\over w_{\min}}\right)\nonumber\\
\stackrel{(b)}{=}&{|\sh_2-\sh_1| \over s} \left( 1 +{1\over 2\gamma} \ln {1-w_{\min}\over w_{\min}}+O(\epsilon)\right),\label{eq:bound-varepsilon}
\end{align}
where $(a)$ follows from \eqref{eq:bd-tau-i}, \eqref{eq:term2-bd}, \eqref{eq:term3-bd} and \eqref{eq:term6-bd}, and $(b)$ holds because $\dmu^2{s}=4\gamma$. Also, note that, from \eqref{eq:bd-tau-i},
\begin{align}
|\varepsilon|\leq& 8\gamma \left( 2\gamma + \ln {1-w_{\min}\over w_{\min}}+O(\epsilon) \right)\epsilon.
\end{align}
Therefore, if $(16\gamma_{\max}^2+8\gamma_{\max}\ln{1-w_{\min})\over w_{\min}}+2\gamma_{\max}\epsilon)\epsilon<{1\over 2}$,
\[
|\varepsilon|\leq {1\over 2},
\] and $|f''(r)|={1\over 4}(1+r)^{-{3\over 2}}\leq {1\over \sqrt{2}}<1$.
Combining \eqref{eq:bound-varepsilon} with \eqref{eq:diff-t-opt-t-tilde-level2}, it follows that
\begin{align}
\left|\tilde{t}_1-t_{\rm opt}+\mu_1\right|  \leq&   \left( \dmu +{2\over w_{\min}\dmu s}+\left({2\over \dmu s}+\dmu\right) \ln {1-w_{\min}\over w_{\min}}+{16\gamma_{\max}^2   \over \dmu s }\right) \epsilon\nonumber\\
&+|f''(r)||\Dhmu\sh_2|  |\sh_2-\sh_1| \left(  1+{1\over 2\gamma}\ln {1-w_{\min}\over w_{\min}}+O(\epsilon)\right)^2+o(\epsilon)\nonumber\\
\leq&   \left( \dmu +{2\over w_{\min}\dmu s}+\left({2\over \dmu s}+\dmu\right) \ln {1-w_{\min}\over w_{\min}}+{16\gamma_{\max}^2   \over \dmu s }\right) \epsilon\nonumber\\
&+|f''(r)||\delta_{\mu} s|(1+\epsilon)(16\gamma^2\epsilon)   \left(  1+{1\over 2\gamma}\ln {1-w_{\min}\over w_{\min}}+O(\epsilon)\right)^2+o(\epsilon)\nonumber\\
\leq&   \left( \dmu +{2\over w_{\min}\dmu s}+\left({2\over \dmu s}+\dmu\right) \ln {1-w_{\min}\over w_{\min}}+{16\gamma_{\max}^2   \over \dmu s }\right.\nonumber\\
&\;\;\;\;\left.+|f''(r)||\delta_{\mu} s|   \left(  4\gamma+2\ln {1-w_{\min}\over w_{\min}}\right)^2\right) \epsilon+o(\epsilon).\label{eq:final-bd-t-tilde1-t-opt}
\end{align}
Dividing both sides of \eqref{eq:final-bd-t-tilde1-t-opt} by $\sigma_1$, and noting that $\dmu/(2\sigma_1)=\gamma$ and $|f''(r)|\leq 1$, we derive
\begin{align}
{\left|\tilde{t}_1-t_{\rm opt}+\mu_1\right| \over \sigma_1}
\leq&   \left( 2\gamma +{1\over w_{\min}\gamma}+\left({1\over \gamma}+2\gamma\right) \ln {1-w_{\min}\over w_{\min}}+{8\gamma_{\max}^2   \over \gamma }\right.\nonumber\\
&\;\;\;\;\left.+{2\gamma}   \left(  4\gamma+2\ln {1-w_{\min}\over w_{\min}}\right)^2\right) \epsilon+o(\epsilon).\label{eq:final-bd-t-tilde-1}
%\leq&   \left( 2c_{\max} +{1\over w_{\min}c_{\min}}+\left({1\over c_{\min}}+2c_{\max}\right) \ln {1-w_{\min}\over w_{\min}}+{8c_{\max}^2   \over c_{\min} }\right.\nonumber\\
%&\;\;\;\;\left.+{2c_{\max}f''(r)}   \left(  4c_{\max}+2\ln {1-w_{\min}\over w_{\min}}\right)^2\right) \epsilon+\epsilon_3.
\end{align}

Finally, as a reminder, from \eqref{eq:t-tilde2}, $\tilde{t}_2={2\Dhmu\sh_2 \over \sh_2-\sh_1}-\tilde{t}_1$. From \eqref{eq:bd-tau-i},  $ |\sh_2-\sh_1|\leq 16\gamma^2s\epsilon$. Hence, if $\hat{\sigma}_1\leq \hat{\sigma}_2$,
\[
-{\tilde{t}_2\over \hat{\sigma}_1}={1\over \hat{\sigma}_1}\left({2\Dhmu\sh_2 \over \sh_1-\sh_2}+\tilde{t}_1\right)\geq {1 \over 4 \gamma \epsilon } +o\left({1\over \epsilon}\right).
\]
Similarly, if $\hat{\sigma}_2\leq \hat{\sigma}_1$,
\[
{\tilde{t}_2-\Dhmu\over \hat{\sigma}_2}={1\over \hat{\sigma}_2}\left({2\Dhmu\sh_2 \over \sh_2-\sh_1}-\tilde{t}_1-\Dhmu\right)\geq {1 \over 4 \gamma \epsilon } +o\left({1\over \epsilon}\right).
\]
Combining \eqref{eq:final-bd-t-tilde-1} and the above equations with \eqref{eq:e-min-e-opt-1} and \eqref{eq:e-min-e-opt-2} yields the desired result.
 \end{proof}

%\end{thebibliography}

\begin{thebibliography}{1}
\bibitem{Achiloptas}
D.~Achlioptas and F.~McSherry.
\newblock On spectral learning of mixtures of distributions.
\newblock In {\em Proc. Ann. Conf. on Learn. Theory (COLT)}, pages 458--469,
  Jun. 2005.

\bibitem{Sinha1}
M.~Belkin and K.~Sinha.
\newblock Toward learning gaussian mixtures with arbitrary separation.
\newblock In {\em Proc. Ann. Conf. on Learn. Theory (COLT)}, pages 407--419,
  Jun. 2010.

\bibitem{Sinha2}
M.~Belkin and K.~Sinha.
\newblock Polynomial learning of distribution families.
\newblock {\em {SIAM} J. Comput.}, 44(4):889--911, 2015.

\bibitem{Dasgupta:99}
S.~Dasgupta.
\newblock Learning mixtures of {Gaussians}.
\newblock In {\em 40th Annual Symp. on Found. of Comp. Sci.}, pages 634--644,
  Oct. 1999.

\bibitem{Dasgupta_concntrt}
S.~Dasgupta, D.~J. Hsu, and N.~Verma.
\newblock A concentration theorem for projections.
\newblock In {\em UAI}, 2006.

\bibitem{schulman}
S.~Dasgupta and L.~J. Schulman.
\newblock A two-round variant of {EM} for gaussian mixtures.
\newblock In {\em Proc. of the Conf. on Uncert. in Art. Int.}, UAI'00, pages
  152--159, 2000.

\bibitem{Dempster}
A.~P. Dempster, N.~M. Laird, and D.~B. Rubin.
\newblock Maximum likelihood from incomplete data via the {EM} algorithm.
\newblock {\em J. of Roy. Stat. Soc., Ser. B}, 39(1):1--38, 1977.

\bibitem{Feldman}
J.~Feldman, R.~A. Servedio, and R.~O'Donnell.
\newblock {PAC} learning axis-aligned mixtures of gaussians with no separation
  assumption.
\newblock In {\em Proc. Ann. Conf. on Learn. Theory (COLT)}, pages 20--34, Jun.
  2006.

\bibitem{price}
M.~Hardt and E.~Price.
\newblock Tight bounds for learning a mixture of two {Gaussians}.
\newblock In {\em ACM Symp. on Theory of Comp.}, pages 753--760. ACM, 2015.

\bibitem{JalaliM:11}
S.~Jalali and A.~Maleki.
\newblock Minimum complexity pursuit.
\newblock In {\em Proc. 49th Annual Allerton Conf. on Commun., Control, and
  Comp.}, pages 1764--1770, Sep. 2011.

\bibitem{kalai2010efficiently}
A.~T. Kalai, A.~Moitra, and G.~Valiant.
\newblock Efficiently learning mixtures of two {Gaussians}.
\newblock In {\em ACM Symp. on Theory of Comp.}, pages 553--562. ACM, 2010.

\bibitem{Kenan:05}
R.~Kannan, H.~Salmasian, and S.~Vempala.
\newblock The spectral method for general mixture models.
\newblock In {\em Proc. Ann. Conf. on Learn. Theory (COLT)}, pages 444--457,
  Jun. 2005.

\bibitem{uci}
M.~Lichman.
\newblock {UCI} machine learning repository, 2013.

\bibitem{LloydKmeans}
S.~Lloyd.
\newblock Least squares quantization in pcm.
\newblock {\em IEEE Trans. Inf. Theor.}, 28(2):129--137, Sept. 1982.

\bibitem{Moitra}
A.~Moitra and G.~Valiant.
\newblock Settling the polynomial learnability of mixtures of gaussians.
\newblock In {\em Proceedings of the 2010 IEEE 51st Annual Symposium on
  Foundations of Computer Science}, FOCS '10, pages 93--102, Washington, DC,
  USA, 2010. IEEE Computer Society.

\bibitem{pearson1894contributions}
K.~Pearson.
\newblock Contributions to the mathematical theory of evolution.
\newblock {\em Phil. Trans. R. Soc. Lond. A}, 1894.

\bibitem{Arora}
A.~Sanjeev and R.~Kannan.
\newblock Learning mixtures of arbitrary gaussians.
\newblock In {\em Proceedings of the Thirty-third Annual ACM Symposium on
  Theory of Computing}, STOC '01, pages 247--257, 2001.

\bibitem{VempalaWang}
S.~Vempala and G.~Wang.
\newblock A spectral algorithm for learning mixture models.
\newblock {\em J. Comput. Syst. Sci.}, 68(4):841--860, 2004.

\end{thebibliography}
\end{document}